\documentclass{article}

\usepackage{microtype}
\usepackage{graphicx}
\usepackage{subfigure}
\usepackage{booktabs} %
\usepackage{enumitem}

\usepackage{hyperref}

\usepackage[accepted]{icml2023}

\usepackage{amsmath}
\usepackage{amssymb}
\usepackage{mathtools}
\usepackage{amsthm}
\usepackage{dsfont}
\usepackage{caption}

\usepackage[capitalize,noabbrev]{cleveref}

\theoremstyle{plain}
\newtheorem{theorem}{Theorem}[section]

\theoremstyle{definition}

\theoremstyle{remark}
\newtheorem{remark}[theorem]{Remark}

\usepackage[textsize=tiny]{todonotes}

\icmltitlerunning{Submission and Formatting Instructions for ICML 2023}

\begin{document}

\twocolumn[
\icmltitle{Accuracy on the Curve: On the Nonlinear Correlation of ML Performance Between Data Subpopulations}

\icmlsetsymbol{equal}{*}

\begin{icmlauthorlist}
\icmlauthor{Weixin Liang}{equal,cs,ee}
\icmlauthor{Yining Mao}{equal,ee}
\icmlauthor{Yongchan Kwon}{equal,columbia}
\icmlauthor{Xinyu Yang}{zju,cornell}
\icmlauthor{James Zou}{cs,ee,DBDS}
\end{icmlauthorlist}

\icmlaffiliation{cs}{Department of Computer Science, Stanford University, Stanford, CA, USA}
\icmlaffiliation{ee}{Department of Electrical Engineering, Stanford University, Stanford, CA, USA}
\icmlaffiliation{columbia}{Department of Statistics, Columbia University, New York, NY, USA}
\icmlaffiliation{zju}{Department of Computer Science and Engineering, Zhejiang University, Hangzhou, P.R.China}
\icmlaffiliation{cornell}{Department of Information Science, Cornell University, Ithaca, NY, USA}
\icmlaffiliation{DBDS}{Department of Biomedical Data Science, Stanford University, Stanford, CA, USA}

\icmlcorrespondingauthor{Weixin Liang}{wxliang@stanford.edu}
\icmlcorrespondingauthor{James Zou}{jamesz@stanford.edu}

\icmlkeywords{Machine Learning, ICML, Out-of-distribution performance, Distribution Shifts, Subpopulation Shifts, Moon Shape, Nonlinear Correlation}

\vskip 0.3in ]

\printAffiliationsAndNotice{\icmlEqualContribution} %
\begin{abstract}

Understanding the performance of machine learning (ML) models across diverse data distributions is critically important for reliable applications. Despite recent empirical studies positing a near-perfect linear correlation between in-distribution (ID) and out-of-distribution (OOD) accuracies, we empirically demonstrate that this correlation is more nuanced under subpopulation shifts. Through rigorous experimentation and analysis across a variety of datasets, models, and training epochs, we demonstrate that OOD performance often has a nonlinear correlation with ID performance in subpopulation shifts. Our findings, which contrast previous studies that have posited a linear correlation in model performance during distribution shifts, reveal a "moon shape" correlation (parabolic uptrend curve) between the test performance on the majority subpopulation and the minority subpopulation. This non-trivial nonlinear correlation holds across model architectures, hyperparameters, training durations, and the imbalance between subpopulations. Furthermore, we found that the nonlinearity of this "moon shape" is causally influenced by the degree of spurious correlations in the training data. Our controlled experiments show that stronger spurious correlation in the training data creates more nonlinear performance correlation. We provide complementary experimental and theoretical analyses for this phenomenon, and discuss its implications for ML reliability and fairness. Our work highlights the importance of understanding the nonlinear effects of model improvement on performance in different subpopulations, and has the potential to inform the development of more equitable and responsible machine learning models.

\end{abstract}

\begin{figure}[ht]
    \centering
    \includegraphics[width=0.48\textwidth]
    {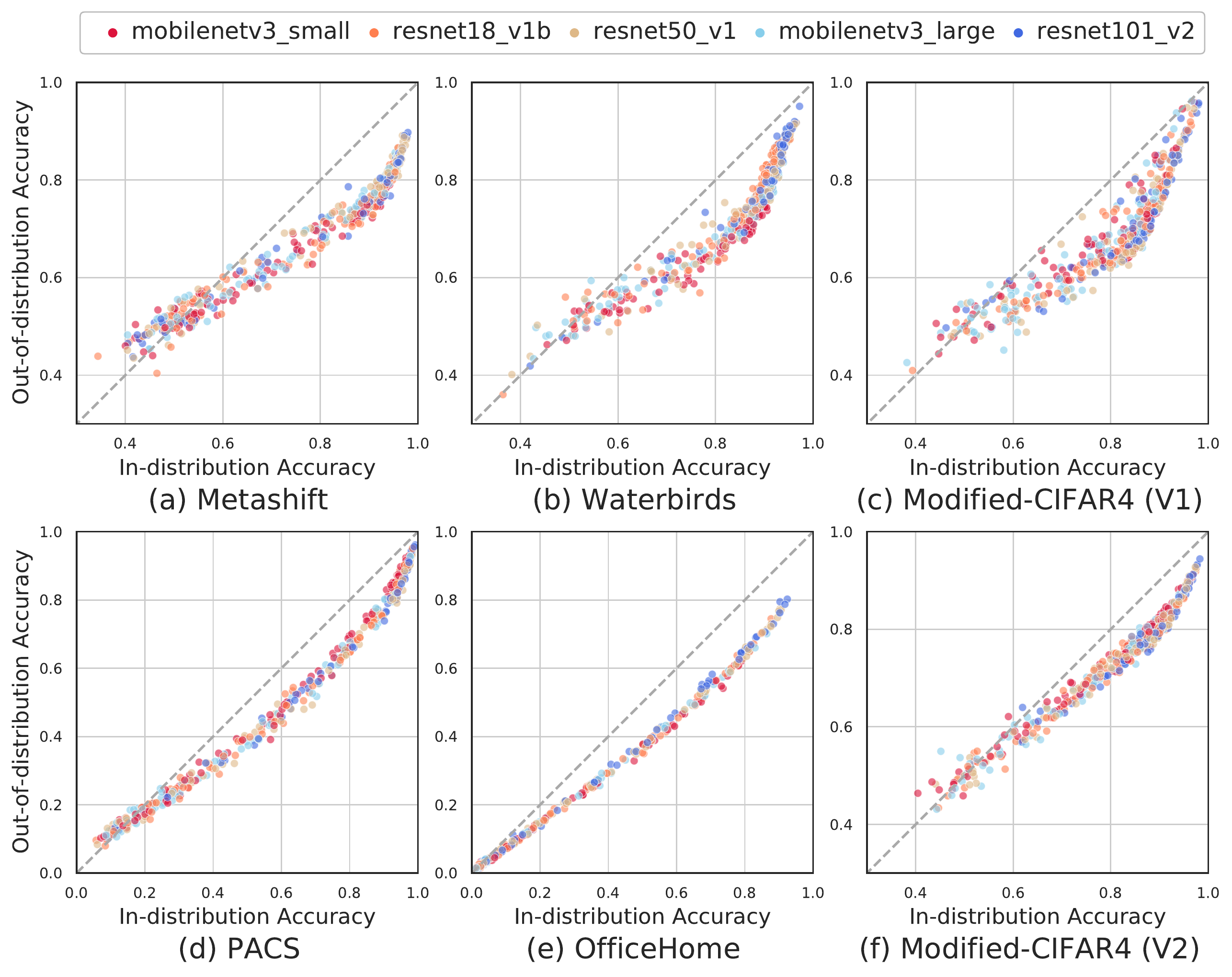}
    \caption{
    \small
\textbf{A striking nonlinear correlation between out-of-distribution and in-distribution performance under subpopulation shifts.}
Each dot represents a trained model and each panel represents a dataset.
Our comprehensive experimentation, utilizing a variety of model architectures and hyperparameters, reveals a precise correlation between OOD and ID performance.
The top panels (a-c) depict datasets constructed with spurious correlations, where the correlation is notably nonlinear.
The bottom panels (d-f) depict datasets with rare subpopulations (absent of spurious correlations), where the correlation is more subtle, but as our analysis in Figure~\ref{fig:figure2}, illustrates, still nonlinear. 
Our findings have significant implications for understanding and improving the reliability and fairness of machine learning models.
    }
    \label{fig:figure1}
\end{figure}

\section{Introduction}

Machine learning (ML) models often exhibit vastly different performance and behaviors when applied to different data distributions. This can be a significant challenge in ML, as even the best efforts to create data that closely represents the real-world may not fully capture the dynamic, high-dimensional, and combinatoric complexity of many tasks. As a result, AI models deployed in the wild are likely to encounter out-of-distribution (OOD) data, raising fundamental questions such as, "Can you trust your model on out-of-distribution data? How does an AI model's in-distribution (ID) performance relate to its out-of-distribution (OOD) performance?"

Exciting progress has been made in addressing these questions on OOD performance, but classical domain adaptation theory only provides a partial answer. Built on the uniform-convergence framework, these classical works resort to bounding OOD performance by quantifying the distance between the ID and OOD (e.g., via the $\mathcal{H} \Delta \mathcal{H}$ divergence~\citep{ben2010theory}), thereby producing an upper bound on OOD performance that becomes increasingly loose for larger distribution shifts~\citep{redko2020survey}.

Pioneering recent research, such as that by \citet{robustness-natural-shifts,accuracy-on-the-line,kaplun2022deconstructing}, has uncovered an exciting phenomenon that is not captured by classical theory: an almost perfect linear correlation in probit scale between ID and OOD performance, which has been repeatedly found across a wide spectrum of OOD benchmarks. This phenomenon, known as "accuracy-on-the-line," holds not only for dataset reconstruction shifts (e.g., ImageNet-V2~\citep{ImageNet-V2}, CIFAR-10.1~\citep{CIFAR-10.1}), but also for more complex distribution shift benchmarks such as WILDS~\citep{WILDS} and BREEDS~\citep{BREEDS}.

The "accuracy-on-the-line" phenomenon has garnered significant interest and excitement, as it suggests that, given access to the slope and bias of the linear correlation, predicting OOD accuracy becomes straightforward. For example, \citet{agreement} have shown a related "agreement-on-the-line" phenomenon, observing that the OOD agreement between the predictions of any two pairs of neural networks also exhibits a strong linear correlation with their ID agreement. Furthermore, the slope and bias of OOD vs ID agreement closely match that of OOD vs ID accuracy, thereby enabling the prediction of OOD accuracy with just unlabeled data. This has been a long-standing research problem, highlighting the significance of "accuracy-on-the-line".

In this work, we investigate a common and challenging type of distribution shift known as \emph{subpopulation shifts}. This phenomenon is frequently observed in real-world applications, such as medical AI models that perform differently when deployed on different sites with different demographics~\citep{wu2021medical,DDI}.

Our research reveals a \emph{nonlinear} correlation between in-distribution (ID) and out-of-distribution (OOD) performance under subpopulation shifts (as demonstrated in Figure~\ref{fig:figure1}).
To gain insight into this phenomenon, we decompose the model's performance into performance on each subpopulation (as shown in Figure~\ref{fig:figure2}).
We observe a consistent \emph{``moon shape'' correlation} (parabolic uptrend curve) between the test performance on the majority subpopulation and the minority subpopulation.
These nonlinear correlations are observed across a variety of datasets, models, and training epochs (as depicted in Figure~\ref{fig:figure-persist}), and are also present in multi-subpopulation data (as shown in Figure~\ref{fig:3D-moonshape}) and under different distribution shift algorithms (Figure~\ref{fig:OOD-algorithms}).

To understand the underlying causes of this nonlinear performance correlation, we conducted an extensive empirical analysis. Our results indicate that the degree of \emph{spurious correlations} in training data plays a significant role. Spurious correlations refer to connections between variables that appear to be causal but are not~\citep{pearl2000models}. We find that datasets with spurious correlations (as shown in Figure~\ref{fig:figure2} top) exhibit more nonlinear performance correlations than datasets without spurious correlations (as shown in Figure~\ref{fig:figure2} bottom). Our controlled study confirms that stronger spurious correlation leads to more nonlinear performance correlation (as depicted in Figure~\ref{fig:figure-spurious}).

This research highlights an important issue with state-of-the-art AI models, which often pick up spurious correlations and biases in training data~\citep{liang2022advances}. These correlations may initially improve performance, but can fail catastrophically when deployed in slightly different environments. Furthermore, our findings indicate that current agreement-based approaches for predicting OOD performance \emph{systematically overestimate} performance under the presence of spurious correlations (as shown in Figure~\ref{fig:figure-agreement}), suggesting a need for new methods to address this problem.

It is worth emphasizing that this work does not contradict, but rather complements and extends previous work. Our work confirms the existence of \emph{strong} and \emph{precise} correlations between ID and OOD performance, which have not been fully captured by classical domain adaptation theory. Additionally, we have identified the presence of spurious correlation as a contributing factor to this broader performance correlation phenomenon. Although the correlations exhibit nonlinearity, they remain geometrically simple, opening up opportunities for the development of new methods for predicting OOD performance.

Beyond distribution shifts, the significance of our findings can also be viewed in the context of machine learning (ML) reliability and fairness.
It is well-documented that a model can have disparate performances even within different subsets of its training and evaluation data~\citep{eyuboglu2022domino,liang2023gpt}.
Furthermore, these performance disparities can have a cascading effect, leading to decreased user retention and further amplifying the performance gap over time~\citep{Fairness-Without-Demographics,ML-Credit-Markets}.
For example, computer-vision AI models for diagnosing malignant skin lesions performed substantially worse on lesions appearing on dark skin compared to light skin, with the area under the receiver operating curves (AUROC) dropping by 10-15\% across skin tones~\citep{DDI}.
Our work shows that ML performances between data subpopulations, albeit disparate, can have much more precise correlations than previously expected from the literature.
Furthermore, we also identify certain situations in the presence of spurious correlations where performance improvement for the majority subpopulation leads to \emph{alarmingly consistent performance deterioration} for the minority subpopulation.
As the existence of ML performance disparities across subpopulations sternly undermines the trustworthiness, reliability, and fairness of ML models, our work makes a critical step towards the empirical understanding of \textit{how ML performances between data subpopulations are correlated}.
Concretely, this paper makes the following main \textbf{contributions}:
\begin{itemize}
    \itemsep0.2em
    
    \item 
    To the best of our knowledge, we present the first systematic study on the performance correlation between data subpopulations. We found a nonlinear, \emph{``moon shape''} correlation between the test performance on the \emph{majority} subpopulation and the \emph{minority} subpopulation (Figure~\ref{fig:figure2}). This indicates that ML performances between data subpopulations, albeit disparate, can have much more precise correlations than previously expected from the literature.

    \item 
    In contrast to recent works reporting a near-perfect linear correlation, we found a nonlinear correlation under subpopulation shifts between ID and OOD accuracies (Figure~\ref{fig:figure1}). We empirically show that this non-trivial \emph{nonlinear} correlation holds across model architectures, hyperparameters, training durations, and the imbalance between subpopulations (Figure~\ref{fig:figure-persist}). In addition, our experiments on multi-subpopulation datasets and different distribution shift algorithms beyond ERM further highlights the generality of this phenomenon. We also demonstrate how our findings complement and contrast previous empirical studies under probit-transformed axes (Figure~\ref{fig:figure-probit}). Our findings significantly broaden the scope of the broader correlation phenomenon between ID performance and OOD performance. 
    
    \item 
    Supported by extensive empirical and theoretical analysis (see Appendix E for detailed theoretical analysis), we identify the degree of \emph{spurious correlations} in training data as an important cause for this nonlinearity. We demonstrate that datasets with spurious correlations (Figure~\ref{fig:figure2} top) show \emph{more nonlinear} correlations than datasets without spurious correlations (Figure~\ref{fig:figure2} bottom). We conduct rigorous controlled studies confirming that stronger spurious correlations create more nonlinear performance correlations (Figure~\ref{fig:figure-spurious}).

    \item 
    We demonstrate that current agreement-based approach for predicting OOD performance would \emph{systematically overestimate} under the existence of spurious correlations (Figure~\ref{fig:figure-agreement}). We also identify certain regimes where performance improvement for the majority subpopulation leads \emph{alarmingly} to \emph{consistent performance deterioration} for the minority subpopulation, thereby highlighting the significance and implications of our findings for ML reliability and fairness.

\end{itemize}

\begin{figure}[ht]
    \centering
    \includegraphics[width=0.48\textwidth]
    {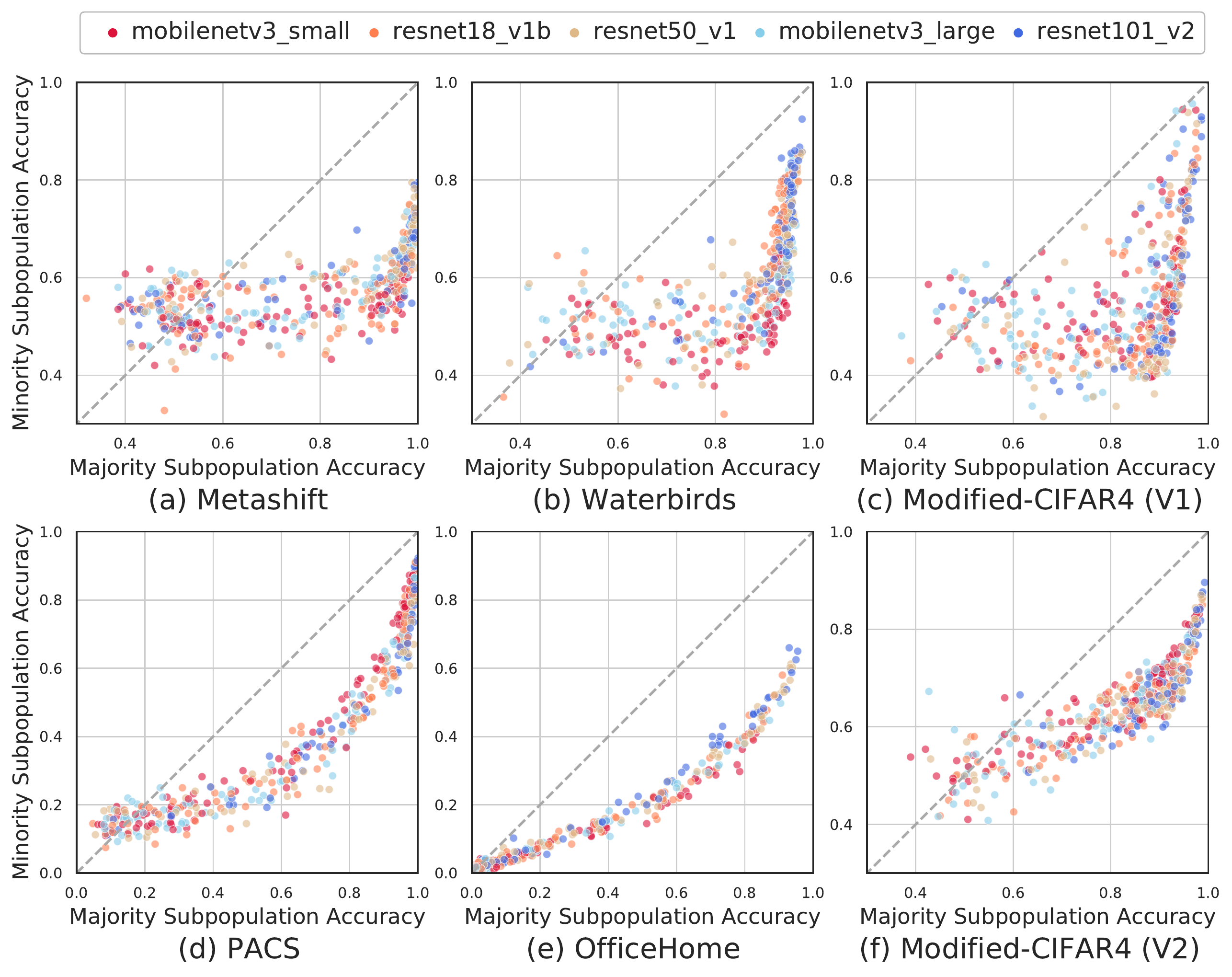}
    \caption{
    \small
    \textbf{A performance breakdown of Figure 1: majority subpopulation vs. minority subpopulation.}
    To gain a deeper understanding of the striking nonlinear correlation between out-of-distribution (OOD) and in-distribution (ID) accuracies observed in Figure 1, we decompose the model's performance into separate evaluations on the majority and minority subpopulations of the OOD test set. 
    \textbf{The results reveal a clear and striking nonlinear correlation, which we term the “\emph{moon shape}” correlation, between the majority subpopulation performance and the minority subpopulation performance.} 
    This nonlinearity is particularly pronounced in datasets constructed with spurious correlations, as seen in the top panels (a-c), while datasets without such correlations, shown in the bottom panels (d-f), exhibit more subtle nonlinearity. 
    }
    \label{fig:figure2}
\end{figure}

\section{Experimental Setup}

\paragraph{Preliminaries: Machine Learning with Diverse Subpopulations}
In our study, we investigate the performance of various machine learning (ML) models in the presence of diverse subpopulations in the data distribution. Specifically, the overall data distribution is denoted as $\mathcal{D}={1,\ldots,D}$, where each subpopulation $d\in \mathcal{D}$ corresponds to a fixed data distribution $P_d$. In our main experiments, we compare the performance of ML models on two different data distributions: (1) the in-distribution (ID), or the training distribution, $P^{tr}=\sum_{d\in \mathcal{D}}r_d^{tr} P_d$, where ${r_d^{tr}}$ denotes the mixture probabilities in the training set, and (2) the out-of-distribution (OOD), which is also a mixture of the $D$ subpopulations, $P^{ts}=\sum_{d\in \mathcal{D}}r_d^{ts} P_d$, where ${r_d^{ts}}$ is the mixture probabilities in the test distribution, but with a different proportion from the training distribution, i.e., $r^{ts}_d \neq r^{tr}_d$ for some $d \in \mathcal{D}$. This setting is known as subpopulation shifts in the literature~\citep{huaxiu,WILDS}.

\paragraph{Experimental Procedure} 
We consider $D=2$ subpopulations for simplicity. 
For the in-distribution (training distribution), we consider a dominating \emph{majority subpopulation} (e.g., $r_{d}^{tr} = 90\%$) and an underrepresented \emph{minority subpopulation} (e.g., $r_{d}^{tr} = 10\%$). 
As for the out-of-distribution, the majority subpopulation and minority subpopulation are equally representative (e.g., $r_{d}^{ts} = 50\%$). 
The goal of our paper is to 
compare the ID performance and the OOD performance across a wide spectrum of ML models. 
Therefore, for each subpopulation shifts dataset, our experimental procedure follows two steps: 
\begin{enumerate}%
\itemsep0.2em
    
    \item 
    Train \emph{500} different ML models independently on the same training set drawn from the training distribution $P^{tr}$ using the empirical risk minimization (ERM) method by varying the model architectures, training durations, and hyperparameters following the search space of commercial AutoML~\citep{AutoGluon}. Details of the training process can be found in the appendix.

    \item 
    For each trained ML model, evaluate the \emph{ID performance} on a test set of held-out samples from the training distribution $P^{tr}$, and the \emph{OOD performance} on a test set drawn from the out-of-distribution $P^{ts}$. We then visualize the correlation of ID and OOD performance on a scatter plot.
    
\end{enumerate}

\paragraph{Subpopulation Shift Datasets}
Based on our survey on the reported cases of ML performance disparity on the minority subpopulation in the wild and prior work~\citep{Domino,Oakden-Rayner2019-zv}, we identified and evaluated on two important categories of subpopulation shift datasets (Figure~\ref{fig:figure_dataset}): 
\begin{itemize}

    \item 
    \textbf{Spurious correlation.} 
    In statistics, a spurious correlation refers to a connection between two variables that appear to be causal, but are not. For example, Figure~\ref{fig:figure_dataset} (a) illustrates a scenario where cat images are mostly indoor and dog images are mostly outdoor (as indicated by the red boxes). A spurious correlation exists between the class labels and the indoor/outdoor contexts. To explore this scenario, we have experimented with three existing datasets in the community: MetaShift~\citep{MetaShift}, Waterbirds~\citep{Sagawa2020-bt}, and Modified-CIFAR4 V1~\citep{Modified-CIFAR-4}.

    \item 
    \textbf{Rare subpopulation.}
    ML models can still underperform on subpopulations that occur infrequently in the training set, even without the presence of obvious spurious correlation. To explore this scenario, we have adopted PACS~\citep{PACS}, OfficeHome~\citep{OfficeHome}, and Modified-CIFAR4 V2~\citep{Modified-CIFAR-4}.

\end{itemize}

\begin{figure}[h]
    \centering
    \includegraphics[width=0.49\textwidth]
    {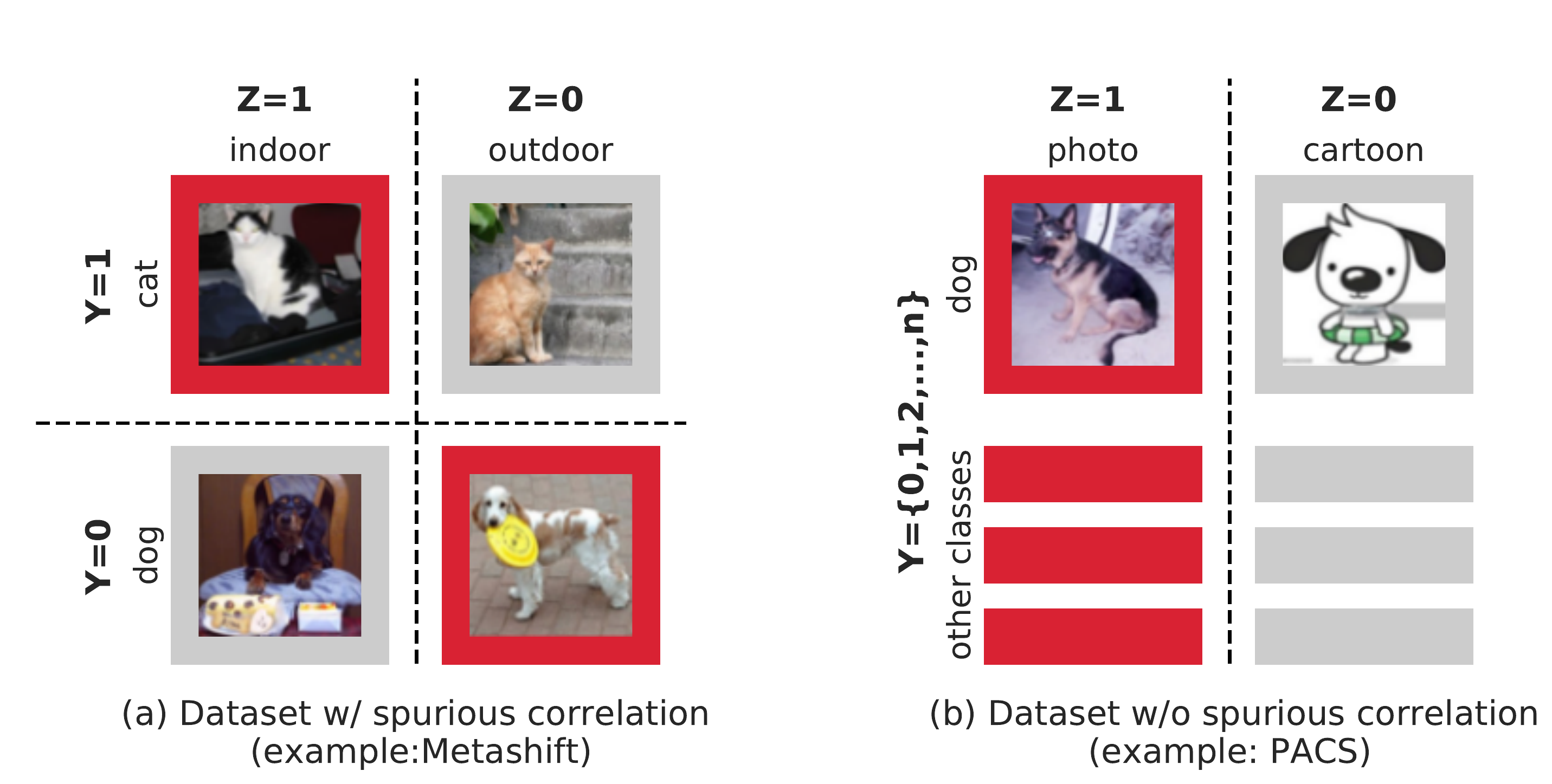}
    \caption{
    \small
    \textbf{Evaluation of subpopulation shift in two dataset configurations: presence or absence of spurious correlations.} As depicted, $Y$ represents the class label, with $Z=1$ and $Z=0$ indicating the majority and minority subpopulations, respectively.
    \textbf{Left:} In the presence of spurious correlation, there exists a correlation between target label $Y$ and $Z$ (e.g. as exemplified by the red boxes, where cat images tend to be predominantly indoor and dog images outdoor);
    \textbf{Right:} Conversely, in the absence of spurious correlation, $Y$ is statistically independent of $Z$. %
    }
    \label{fig:figure_dataset}
\end{figure}

\section{The Moon Shape Phenomenon}
\label{sec:moon-shape-phenomenon}

In this section, we empirically show the \emph{nonlinear} correlation between the ID and OOD performance across multiple subpopulation shifts datasets. 
To understand this phenomenon, we decompose the model's performance into performance on each subpopulation. 
We also found \emph{nonlinear} correlation between the test performance on the \emph{majority subpopulation} and the \emph{minority subpopulation}. 
Moreover, this nonlinearity holds across model architectures, training durations and hyperparameters, and the imbalance between subpopulations. 

\subsection{Nonlinear Correlation of ML Performance Across Data Subpopulations}

\paragraph{Out-of-Distribution vs. In-Distribution} 
Prior research reports a near perfect \emph{linear} correlation between the OOD accuracies and the ID accuracies. 
In contrast, Figure~\ref{fig:figure1} shows the \emph{nonlinear} correlation between the ID and OOD performance across multiple subpopulation shifts datasets. 
Moreover, datasets constructed with spurious correlations (Figure~\ref{fig:figure1} Top) occurs more \emph{nonlinear} compared to 
the datasets with only rare subpopulations (Figure~\ref{fig:figure1} Bottom), which we further confirm and analyze below.

\paragraph{Majority vs. Minority}
To understand this phenomenon, 
we decompose the model's performance into performance on each subpopulation. 
As shown in Figure~\ref{fig:figure2}, there is a \emph{``moon shape'' correlation} (parabolic uptrend curve) between the test performance on the \emph{majority subpopulation} and the \emph{minority subpopulation}. Since we have decomposed by subpopulations, the \emph{nonlinearity} becomes much more visually apparent. 
Figure~\ref{fig:figure2} also confirms that datasets with \emph{spurious correlations} (Figure~\ref{fig:figure2} Top) show more nonlinearity compared to those without (Figure~\ref{fig:figure2} Bottom), which motivates our further analysis on how \emph{spurious correlation} affect the correlation \emph{nonlinearity} ($\S$~\ref{sec:spurious-correlation}).

\textit{Converting from Figure~\ref{fig:figure1} to Figure~\ref{fig:figure2}:} We clarify that the test set is always balanced, with a 50/50 majority to minority ratio when $D=2$. In our subpopulation shift setting, the correlation between "Majority Subpopulation Accuracy vs. Minority Subpopulation Accuracy" can be directly connected to "In-distribution Accuracy vs. Out-of-distribution Accuracy". This is because, in our setting, the in-distribution consists of a mixture of 90\% majority subpopulation and 10\% minority subpopulation, while the test distribution (out-of-distribution) is composed of a 50\% majority subpopulation and 50\% minority subpopulation. Consequently, we have:
\begin{itemize}[leftmargin=0.0cm,itemsep=0.0em]
    \item In-distribution Accuracy = 0.9 $\times$ Majority Subpopulation Accuracy + 0.1 $\times$ Minority Subpopulation Accuracy
    \item Out-of-distribution Accuracy = 0.5 $\times$ Majority Subpopulation Accuracy + 0.5 $\times$ Minority Subpopulation Accuracy
\end{itemize}

\begin{figure}[h]
    \centering
    \includegraphics[width=0.48\textwidth]
    {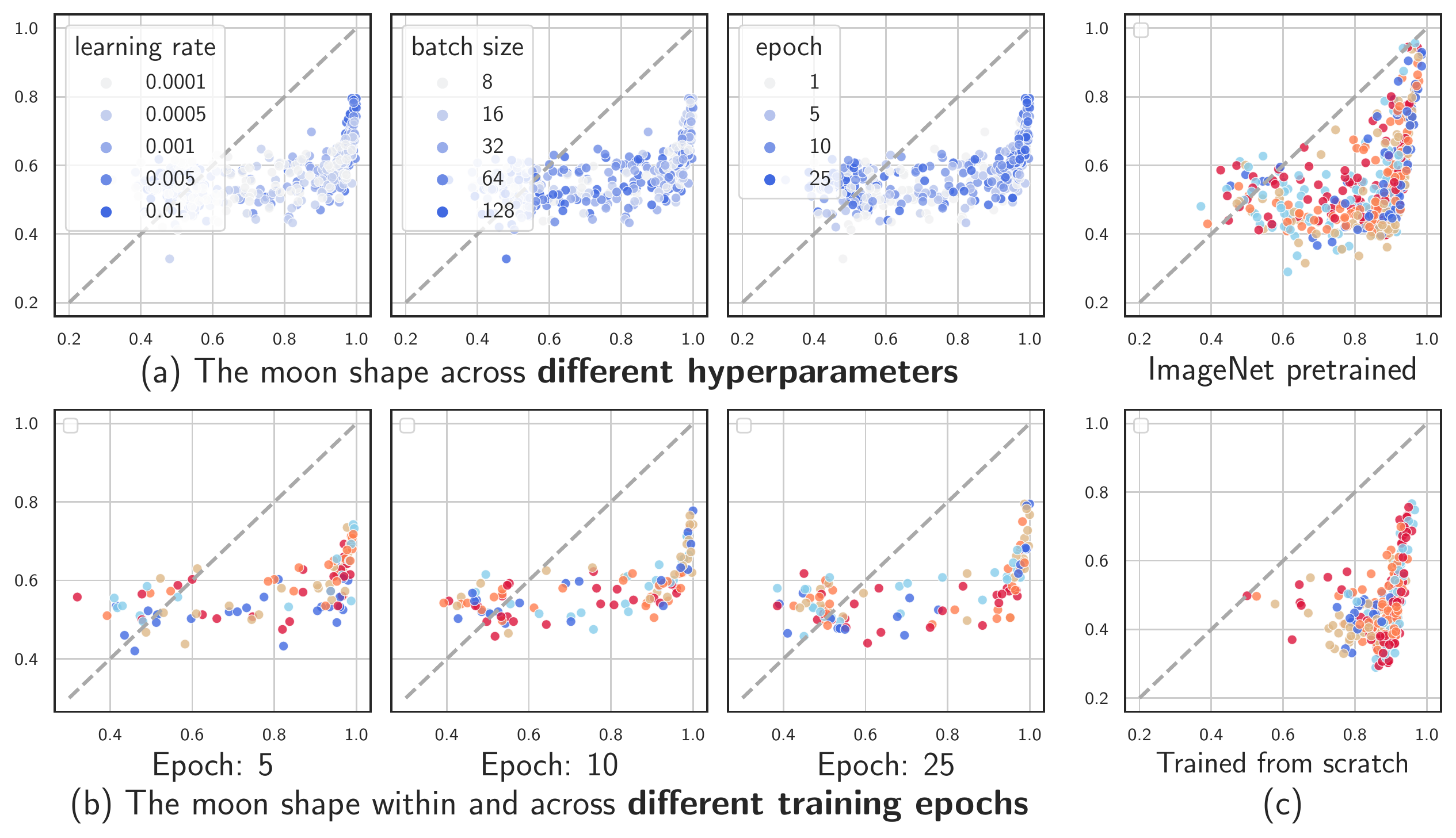}
    \caption{
    \small
    \textbf{Demonstrating the Consistency of the Moon Shape Phenomenon Across Various Factors.}
    The x-axis of each panel represents the performance of the majority subpopulation, while the y-axis represents the performance of the minority subpopulation.
    \textbf{(a)} (Metashift) The moon shape is present regardless of the model architecture, hyperparameters, or training duration utilized. 
    \textbf{(b)} (Metashift) The moon shape is evident at each snapshot and persists across different training epochs.
    \textbf{(c)} (Modified-CIFAR4 (V1)) The moon shape is observed in both models that are pretrained on ImageNet and models that are trained from scratch.
    }
    \label{fig:figure-persist}
\end{figure}

Furthermore, we demonstrate that the moon shape phenomenon holds across various factors including (Figure~\ref{fig:figure-persist}):
\begin{itemize}
\itemsep0.2em
    \item
    \textbf{Model architectures, training durations and hyperparameters (Figure~\ref{fig:figure-persist}\textbf{(a)}).}
    Following the search space of a commercial AutoML library~\citep{AutoGluon}, we varied (1) pretrained model architectures, (2) training durations (i.e., training epochs), (3) hyperparameters such as learning rates and batch size. 
    We found that models lies consistently on the same \emph{``moon shape''}.

    \item
    \textbf{Training dynamics (Figure~\ref{fig:figure-persist} \textbf{(b)}).} 
    We further stratify the dots in Figure~\ref{fig:figure2} (a) (i.e. the trained models) by the number of training epochs. We still find a clear moon shape for each fixed training epoch. %
    Moreover, similar moon shapes persist across different training epochs. Results on other datasets are similar (Supp. Figure~\ref{fig:training-dynamics-other-datasets}). 
    This finding motivates us to focus our analysis on comparing across \emph{different models} rather than comparing the subpopulation performance of a single model across training epochs (which is an interesting direction complementary to our scope).

    \item
    \textbf{Pretrained vs. trained from scratch (Figure~\ref{fig:figure-persist}\textbf{(c)}).}
    Although fine-tuning pretrained models has become a modern paradigm of ML, we also add an experiment of training from scratch, confirming that \emph{moon shape} persists even when training from scratch. 
    This shows that the moon shape is not an artifact of ImageNet pre-trained models, but a much broader phenomenon. 

\end{itemize}

\subsection{Discussion: Why the Moon Shape is not Obvious}

\begin{figure}[h]
    \centering
    \includegraphics[width=0.36\textwidth]
    {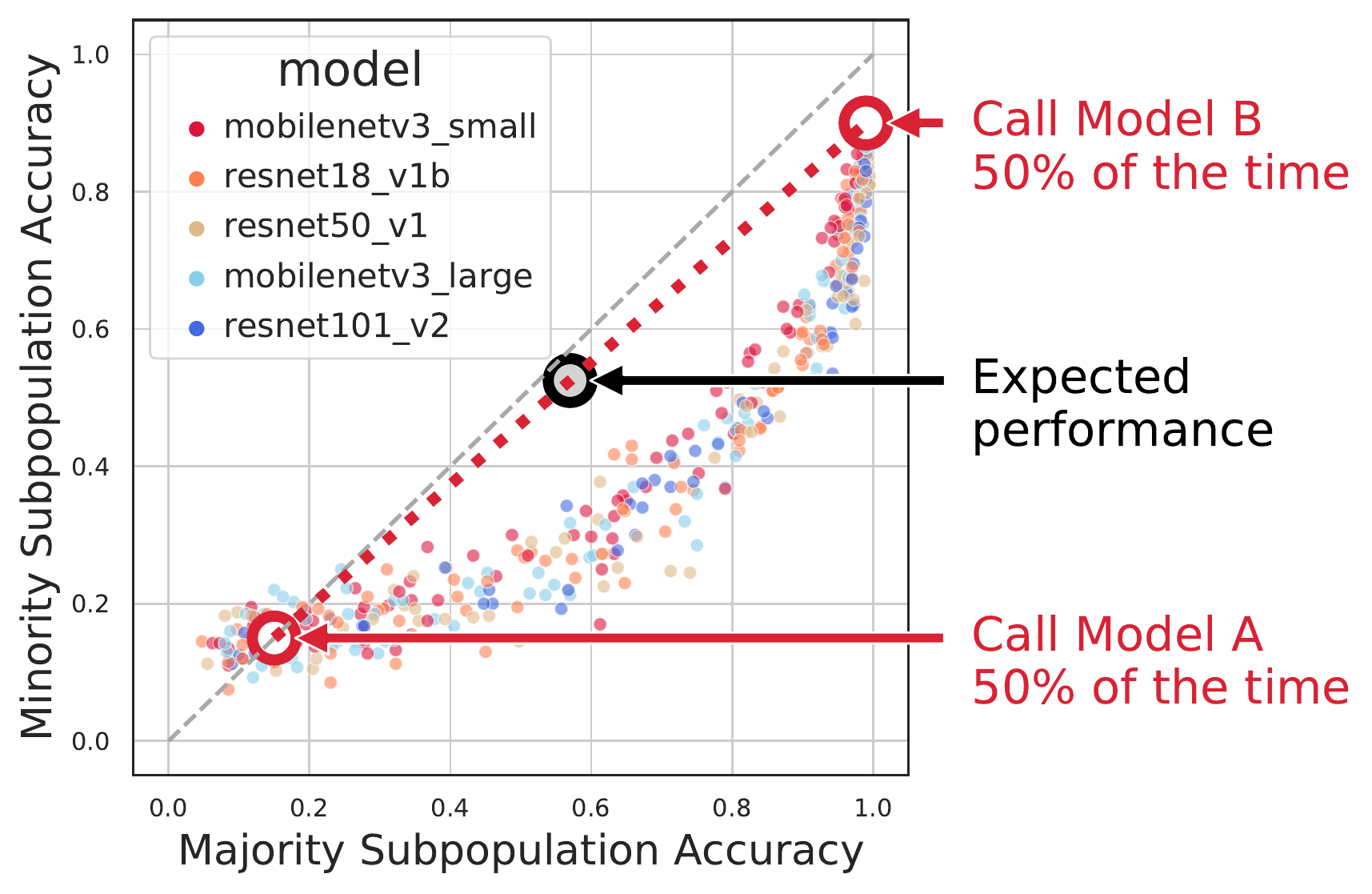}
    \caption{
    \small
    Why the moon shape is not obvious. Mixture of models can fill in the moon shape.
    }
    \label{fig:figure3}
\end{figure}
Figure~\ref{fig:figure3} illustrates one reason why the non-linear correlation structure, represented by the moon shape, is not obvious. A thought experiment is presented in which two models, $A$ and $B$ (indicated by red circles), are interpolated by flipping a biased coin with probability $p$. If the coin lands heads, classification is performed with model $A$, and if it lands tails, classification is performed with model $B$. Varying $p$ in the range $[0,1]$ generates a line between the two models. This interpolation line represents an achievable region for machine learning models, yet our results demonstrate that all models deviate substantially from this line, resulting in the unexpected moon shape.

In summary, the moon shape is intriguing because it contradicts and extends the near-perfect linear correlation reported by previous studies, revealing that performance correlation is more nuanced under subpopulation shifts. Additionally, while it would be expected that the dots of a scatter plot would reside in the \emph{lower triangular area} since machine learning models generally perform worse on the minority subpopulation, our results show that the performance correlation is more concentrated than anticipated, with the dots \emph{concentrated on one curve} rather than spreading out in the lower triangular area.

\begin{figure}[htb] %
    \centering
    \includegraphics[width=0.48\textwidth]{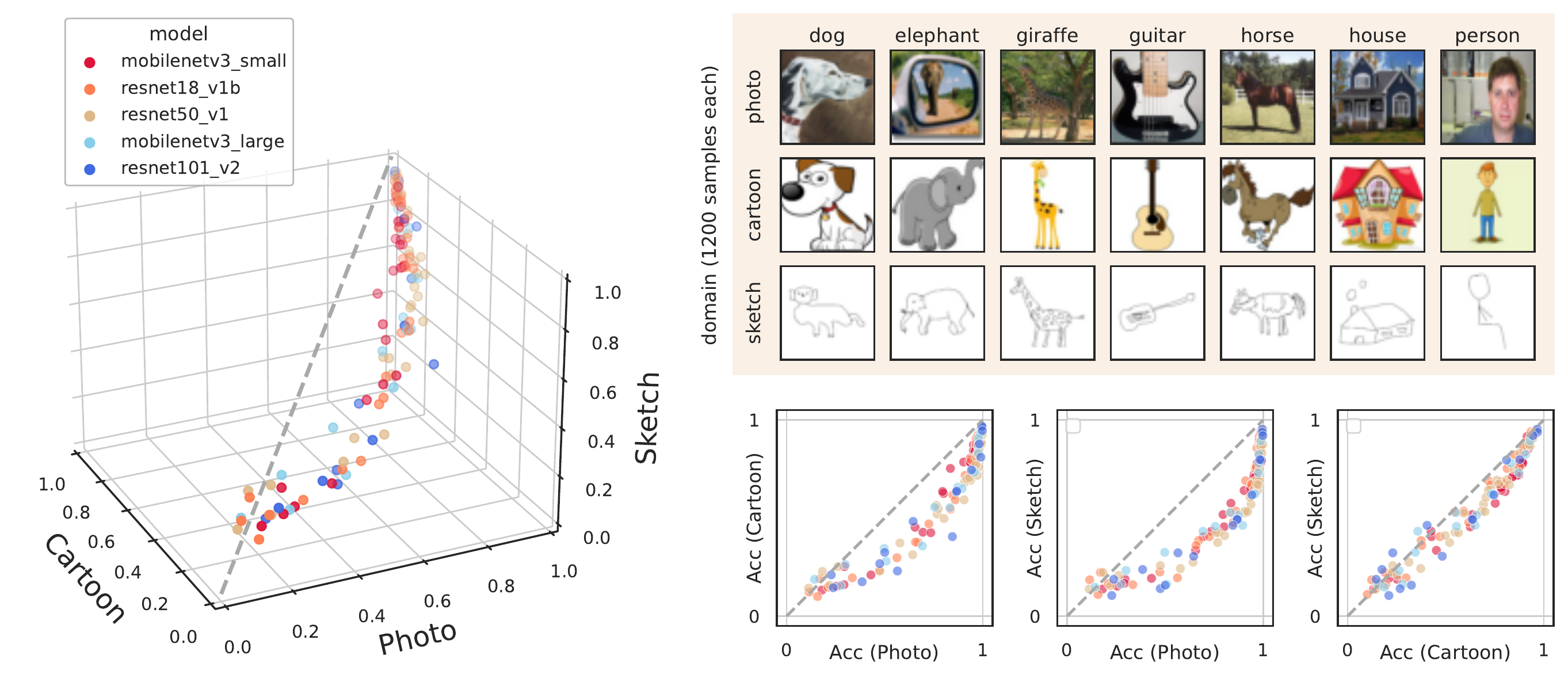}
    \caption{
    \small
    \textbf{Empirical demonstration of the 3D moon shape phenomenon.}
    Results of experiments conducted on the PACS dataset, comprising three subpopulations of images (Cartoon, Photo, and Sketch) are depicted.
    The figure illustrates that the moon shape phenomenon is not limited to just two subpopulations, but extends to three subpopulations as well.
    Importantly, it is demonstrated that the distribution of the number of training samples among the subpopulations is \emph{not} a necessary condition for the emergence of the moon shape phenomenon.
    }
    \label{fig:3D-moonshape}
\end{figure}
\begin{figure}[htb] %
    \centering
    \includegraphics[width=0.48\textwidth]{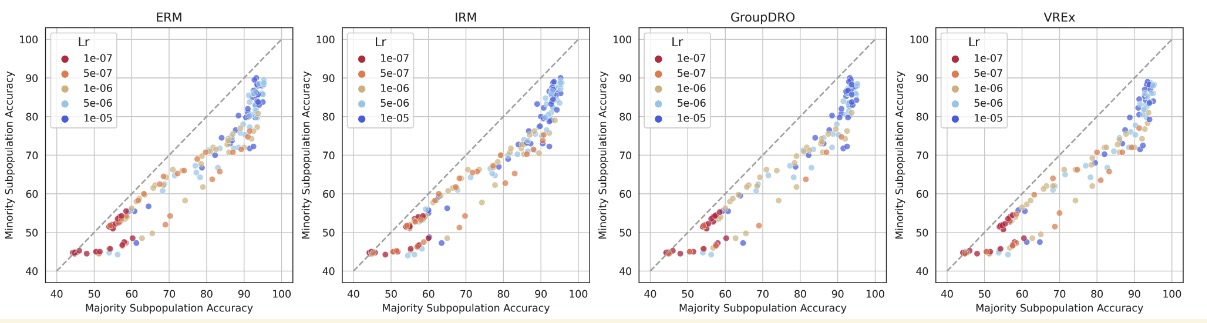}
    \caption{
    \small
\textbf{Generalizability of the moon shape phenomenon beyond Empirical Risk Minimization (ERM).} Our experiments on MetaShift demonstrate that the moon shape phenomenon, previously observed in models trained via ERM, also holds for models trained with various distribution shift algorithms such as Invariant Risk Minimization, GroupDRO, and VREx. This suggests that the "moon shape" is a general phenomenon in machine learning, rather than being specific to ERM.
    }
    \label{fig:OOD-algorithms}
\end{figure}

\subsection{Multi-Subpopulation: 3D Moon Shape}

To establish that the moon shape phenomenon is not limited to just two subpopulations, we conducted an additional experiment on the PACS dataset, which comprises three subpopulations of images (Cartoon, Photo, and Sketch).
The results, depicted in Figure~\ref{fig:3D-moonshape}, clearly exhibit a 3D moon shape, providing empirical evidence for the generality of the moon shape phenomenon beyond two subpopulations.
Furthermore, it is noteworthy that the three subpopulations in our experiment possess an equal number of samples, thereby demonstrating that the \emph{imbalance} of training datasets is \emph{not} a necessary condition for the emergence of nonlinear correlation as represented by the moon shape phenomenon.

\subsection{Generalizability of the Moon Shape Phenomenon to Distribution Shift Algorithms}

Distribution shift remains a significant challenge in machine learning, leading to the development of various algorithms to address it~\citep{JTT}. To explore the generality of the "moon shape" phenomenon, we conducted additional experiments on distribution shift algorithms beyond Empirical Risk Minimization (ERM), including Invariant Risk Minimization (IRM)\citep{IRM}, GroupDRO\citep{GroupDRO}, and VREx~\cite{VREx}. We employed the implementation and hyper-parameter range from DomainBed~\citep{DomainBed} to ensure rigorous and controlled experimental conditions. The results in Figure~\ref{fig:OOD-algorithms} demonstrate that the moon shape phenomenon also holds for IRM, GroupDRO, and VREx, indicating that it is a general phenomenon across different distribution shift algorithms.

\section{The Impact of Spurious Correlation on the Moon Shape}
\label{sec:spurious-correlation}

\subsection{Controlled Experiments: Spurious Correlation Makes the Moon Shape more Nonlinear}

In the previous section (\S~\ref{sec:moon-shape-phenomenon}), we observed that datasets with spurious correlations exhibit increased nonlinearity compared to those without (as shown in Figure~\ref{fig:figure1} and Figure~\ref{fig:figure2}, Top vs Bottom). In this subsection, we conduct well-controlled experiments to quantify the effect of spurious correlation on nonlinearity.

\paragraph{Experiment Design}
We use Modified-CIFAR4 (V1), a subset of the bird, car, horse, and plane classes from CIFAR-10 created by \cite{Modified-CIFAR-4}, as our fixed dataset. We manipulate the degree of spurious correlation between the classification target label (air/land) and the spurious feature (vehicle/animal) in the training data by altering the mixture weights, while maintaining a fixed number of training data points in total (10,000), for each class (5,000), and for each spurious feature (6,000 for vehicle, 4,000 for animal). Additional details can be found in the appendix.

\paragraph{Results and Analysis}
Figure~\ref{fig:figure-spurious} illustrates the results of our experiments with increasingly stronger levels of spurious correlation. As predicted, nonlinearity in performance correlation increases with stronger levels of spurious correlation in the training data. These findings demonstrate that the presence of spurious correlation plays a significant role in shaping the relationship between out-of-distribution and in-distribution performance, an aspect that has been previously overlooked in the literature.

\begin{figure*}[tb]
    \centering
    \includegraphics[width=0.88\textwidth]
    {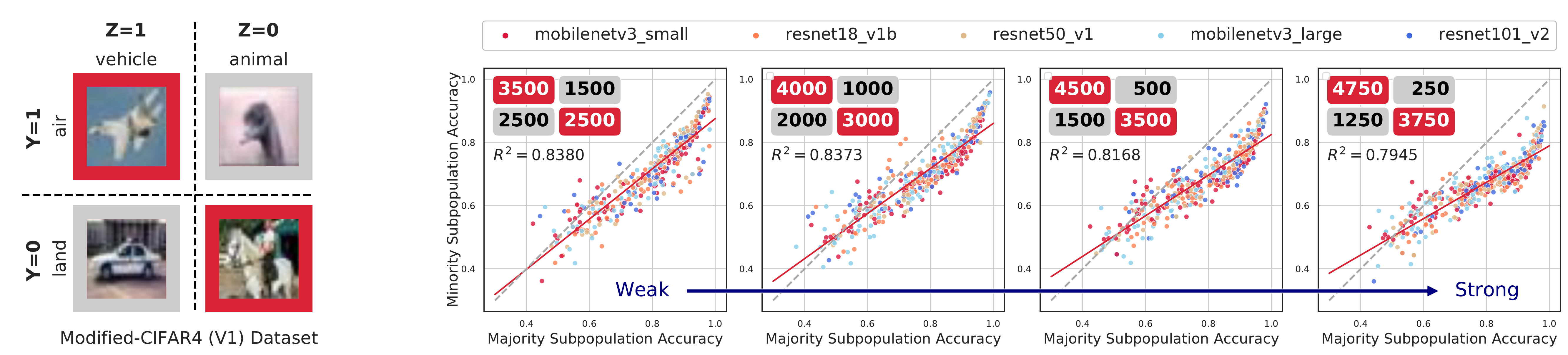}
    \caption{
    \small
\textbf{Stronger spurious correlation creates more nonlinear performance correlation.} The left panel depicts binary classification on Modified-CIFAR4 (V1) with two subpopulations: a majority subpopulation (Z=1) and a minority subpopulation (Z=0). The right panel illustrates that as spurious correlation increases (i.e., more samples in the red boxes), nonlinear performance correlation also increases. The four panels represent different training data, with the number of training samples indicated by the 2x2 tables. The total number of training samples (10,000 images) and the majority:minority ratio (6,000:4,000) are held constant, with evaluation data also fixed.    
    }
    \label{fig:figure-spurious}
\end{figure*}

\begin{figure}[ht]
    \centering
    \includegraphics[width=0.48\textwidth]
    {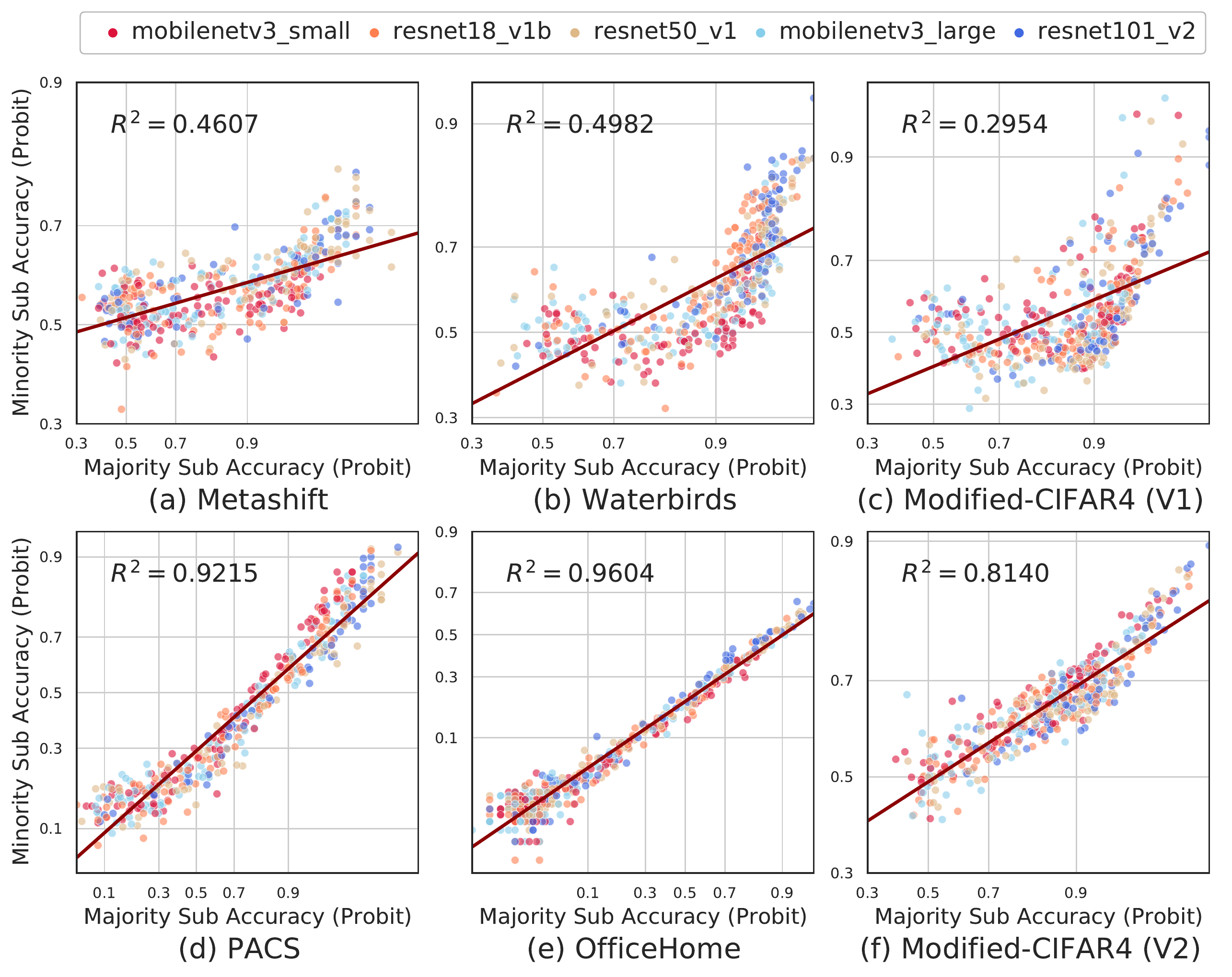}
    \caption{
    \small
\textbf{Probit-scaled comparison of majority and minority subpopulation accuracy.} The top row (a-c) illustrates that, in datasets with spurious correlation, the performance correlation remains nonlinear even under probit transformation. In contrast, the bottom row (d-f) demonstrates that, in the absence of spurious correlation, a linear correlation emerges under probit scaling, as previously reported in literature. Linear fits and corresponding $R^2$ values are included for reference.
    }
    \label{fig:figure-probit}
\end{figure}
\begin{figure}[ht]
    \centering
    \includegraphics[width=0.48\textwidth]
    {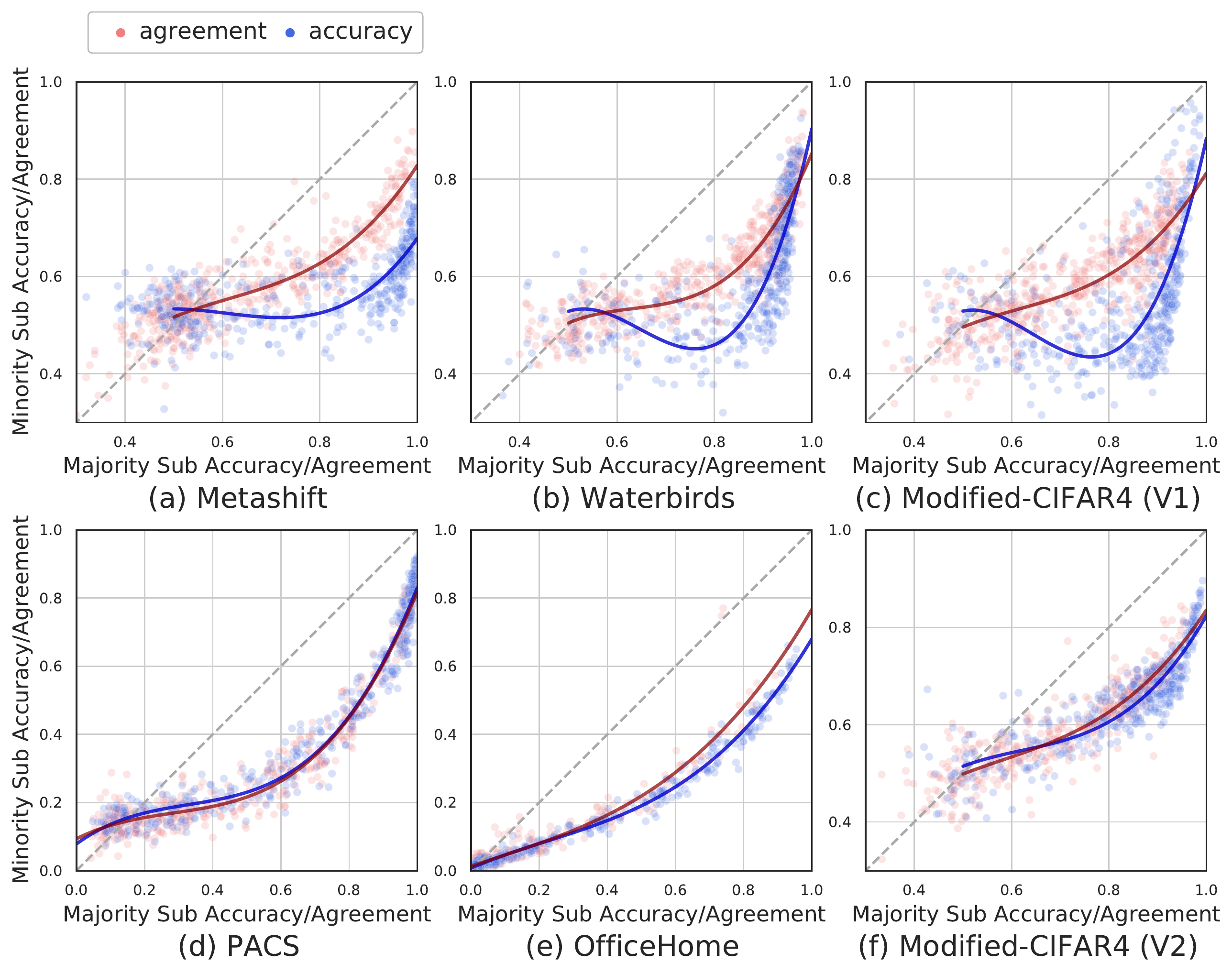}
    \caption{
    \small
\textbf{Model agreements and accuracies in the presence and absence of spurious correlation.} Panels (a-c) demonstrate the significant disparity between the two measures when spurious correlation is present, while panels (d-f) reveal near-perfect alignment between them in the absence of such correlation. To prevent overcrowding, only 500 randomly chosen model agreement pairs with the 500 model accuracies are depicted and both sets of data are smoothed using cubic spline fitting.
    }
    \label{fig:figure-agreement}
\end{figure}

\subsection{Nonlinearity under Probit-Transformed Axes}

Previous research has reported a near-perfect linear correlation between in-distribution and out-of-distribution accuracies, with some studies utilizing probit-transformed axes to enhance the linear trend~\citep{accuracy-on-the-line,robustness-natural-shifts}. In this work, we use the probit transform, which is the inverse of the cumulative density function of the standard Gaussian distribution, to make our results directly comparable to previous studies.

Our findings reveal that, for datasets with spurious correlations, the performance correlation remains nonlinear even under probit transformation (Figure~\ref{fig:figure-probit} top). This confirms that the nonlinear correlations we found are not captured by previous research. However, for datasets without spurious correlations, the performance correlation becomes linear after probit-transformation (Figure~\ref{fig:figure-probit} bottom), as demonstrated by the significant difference in $R^2$ values in the linear fit. This highlights the importance of spurious correlations and how our work complements previous studies.

\subsection{Implications of Model Agreements for Out-of-Distribution Performance Prediction}

\paragraph{Motivation} The ability to accurately predict out-of-distribution (OOD) performance is a valuable application of the \emph{accuracy-on-the-line} phenomenon, as outlined by \citet{agreement}. They propose using model agreements, which are calculated by evaluating the consistency of predictions between pairs of models, as a means of achieving this. They observe a strong linear correlation between OOD and in-distribution (ID) performance for model agreements, referred to as the \emph{agreement-on-the-line} phenomenon. Additionally, they find a precise match between model accuracy and model agreements, which allows for OOD performance prediction using unlabeled data.

The agreement-on-the-line approach proposed by \citet{agreement} reveals a fascinating phenomenon related to the agreement between pairs of neural network classifiers. Specifically, when accuracy-on-the-line holds, a strong linear correlation is observed between out-of-distribution (OOD) agreement and in-distribution (ID) agreement for any two pairs of neural networks, regardless of architectural differences. Notably, the linear trend of ID vs. OOD agreement exhibits the same empirical slope and intercept as the linear trend between ID and OOD accuracy. This discovery, referred to as "agreement-on-the-line," has significant practical implications.

The agreement-on-the-line phenomenon allows for the prediction of OOD accuracy for classifiers without the need for labeled data. OOD agreement can be estimated using only unlabeled data by assessing the disagreement on an unlabeled dataset between pairs of neural networks trained with different sources of randomness. This provides the empirical slope and intercept of ID vs. OOD agreement, which can then be used as the empirical slope and intercept of ID and OOD accuracy.

Inspired by the potential of agreement-based OOD prediction, we conduct experiments on subpopulation shift datasets and present the results using smoothing cubic spline interpolation (Figure~\ref{fig:figure-agreement}).

\paragraph{Results} Our results indicate that the presence of spurious correlation can significantly impact the effectiveness of the agreement-based approach. In the absence of spurious correlation, the agreement results align almost perfectly with accuracy results, consistent with previous work. However, with spurious correlation, the agreement results deviate from accuracy results, forming a distinct moon-shaped pattern that lies above the accuracy moon shape.

This deviation is alarming as it suggests that naively applying the current agreement-based approach for predicting OOD performance would \emph{systematically overestimate} performance in the presence of spurious correlation. This can be understood intuitively as the presence of spurious correlation causes models to rely on a variable correlated with the class label for predictions, breaking the independence of each model's prediction error and leading to increased agreement between models.

\section{Related Work}

\paragraph{Linear Correlations Between ID and OOD Performances}
Existing research mostly reports \emph{linear} correlations between ID and OOD performances. 
The linear correlations were first reported in recent dataset reconstruction settings including ImageNet-V2~\citep{ImageNet-V2}, CIFAR-10.1~\citep{CIFAR-10.1}, CIFAR-10.2~\citep{CIFAR-10.2}, 
where new test sets of popular benchmarks are collected closely following the original dataset creation process. 
As there are subtle differences in the dataset creation pipeline, the test performance on the new test set is often lower, but appears to be linearly correlated with the performance on the original test set~\citep{CIFAR-10.2,SQuAD-dataset-reconstruction,CIFAR-10.1,ImageNet-V2,MNIST-dataset-reconstruction}. 
Later researchers also found the linear trends in the context of cross-benchmark evaluation~\citep{robustness-natural-shifts,accuracy-on-the-line}, 
and transfer learning~\citep{ImageNet-Transfer-Better,OOD-Fine-Tuning}, where a model's ImageNet test accuracy linearly correlates with the transfer learning accuracy. 
Similar linear trends are also observed in sub-type shifts~\citep{ImageNet-C-Dataset,BREEDS} (e.g., the training data for the “dog” class are all from a specific breed while the test data come from another breed). 
Different from these studies, we (1) focus on subpopulation shifts, where we also present the first systematic study on the performance correlation between data subpopulations, and (2) find \emph{nonlinear} correlations of ML performance across data subpopulations, which is not captured in previous work. 
Importantly, we show that for datasets with spurious correlations, the performance correlations still remain nonlinear under probit scale. This confirms that the nonlinear (“moon shape”) correlation phenomenon is indeed not captured by previous work.

\paragraph{ML with Diverse Subpopulations} 
A major challenge in ML is that a model can have very disparate performances even when it is applied to different subpopulations of its training and evaluation data. 
Models with low average error can still fail on particular groups of data points~\citep{Fairness-Without-Demographics,Gender-Shades,Afican-American-English}. 
For example,~\citep{clinical-DRO} reported that predictive models for clinical outcomes which are accurate on average in a patient population underperform drastically for some subpopulations, potentially introducing or reinforcing inequities in care access and quality.
Similar performance disparity has also been observed in 
radiograph classification \citep{Badgeley2019-al, zech2018variable, DeGrave2021-xv}, 
face recognition~\citep{Face-Recognition-Fairness,Gender-Shades}, 
speech recognition~\citep{Racial-disparities-in-automated-speech-recognition,Afican-American-English,language-identification-dialects}, 
academic recommender systems~\citep{academic-recommender-systems-gender}), 
and
automatic video captioning~\citep{Youtube-captioning-Gender-Dialect}, among others. 
As the existence of ML performance disparity across subpopulations sternly undermines the trustworthiness, reliability, and fairness of ML models, our work makes a critical step towards the empirical understanding of how ML performances between data subpopulations are correlated. 
\section{Discussion} 

In this study, we presented a novel and significant discovery of a nonlinear correlation, which we term the “moon shape” phenomenon, between the performance of models on majority and minority data subpopulations. Through meticulous experimentation and analysis across a variety of datasets, models, and training epochs, we demonstrated that this phenomenon is persistent and has far-reaching implications for model selection and performance evaluation. We emphasize the following key aspects regarding the clarification of implications of the moon shape phenomenon:

\textbf{Rethinking Accuracy-on-the-Line:} Our findings present counterexamples to the previously reported linear correlation between in-distribution (ID) and out-of-distribution (OOD) model performance during distribution shifts. We identify the moon shape phenomenon, which exhibits nonlinear yet precise correlations between ID and OOD accuracies, expanding the conventional understanding and offering new insights~\cite{robustness-natural-shifts,accuracy-on-the-line,kaplun2022deconstructing,CIFAR-10.2,SQuAD-dataset-reconstruction,CIFAR-10.1,ImageNet-V2,MNIST-dataset-reconstruction,ImageNet-Transfer-Better}.

\textbf{Implications for OOD Performance Estimation:} Our results reveal that practitioners relying on a linear trend to predict OOD performance may systematically \emph{overestimate} model performance under subpopulation shifts, particularly when spurious correlations are present. This finding suggests that alternative methods are needed to accurately estimate OOD performance. One plausible explanation for this observation is that different ML models, regardless of their performance levels, tend to capture similar spurious correlations in the training data to varying degrees. This breaks the independence of each model's prediction error, leading these models to make similar errors and resulting in higher agreement between them.

\textbf{Impact on Machine Learning Reliability and Fairness:} Our work highlights that ML performance between data subpopulations can display more intricate correlations than previously assumed. Furthermore, we demonstrate situations where, given the presence of spurious correlations, performance improvement for the majority subpopulation coincides with a consistent decline in performance for the minority subpopulation. This observation has significant implications for the reliability and fairness of ML models in real-world applications.

The increasing use of automated machine learning (AutoML) in model building makes selecting models that perform well across diverse subpopulations a paramount challenge. Our results indicate that when there is no spurious correlation, models with higher aggregate performance tend to also perform well on minority subpopulations. However, in the presence of spurious correlation, the situation becomes more complex, with a phase transition point between negative and positive correlation. In settings where subpopulations performance is important (e.g. fairness considerations), we recommend AutoML practitioners to use similar type of scatter plots as our Figure~\ref{fig:figure2} to diagnose model selection. 

Subpopulation shift is a ubiquitous phenomenon in machine learning applications, and our work highlights the importance of understanding the nonlinear effects of model improvement on performance in different subpopulations. Further research and analysis of this nonlinear pattern is a crucial direction for future work, as it has the potential to inform the development of more equitable and responsible machine learning models, and contribute to addressing the fairness considerations in ML.

\section*{Data and Code Availability} 
All original data and code has been deposited at Github under \url{https://github.com/yining-mao/Moon-Shape-ICML-2023} and is publicly available as of the date of publication. 

\section*{Acknowledgements} 
We thank Jonas Mueller, Rasool Fakoor and Shirley Wu for discussions. J.Z. is supported by the National Science Foundation (CCF 1763191 and CAREER 1942926), the US National Institutes of Health (P30AG059307 and U01MH098953) and grants from the Silicon Valley Foundation and the Chan-Zuckerberg Initiative.

\section*{Declaration of Interests} The authors declare no conflict of interest.

\bibliography{main}

\begin{thebibliography}{51}
\providecommand{\natexlab}[1]{#1}
\providecommand{\url}[1]{\texttt{#1}}
\expandafter\ifx\csname urlstyle\endcsname\relax
  \providecommand{\doi}[1]{doi: #1}\else
  \providecommand{\doi}{doi: \begingroup \urlstyle{rm}\Url}\fi

\bibitem[Andreassen et~al.(2021)Andreassen, Bahri, Neyshabur, and
  Roelofs]{OOD-Fine-Tuning}
Andreassen, A., Bahri, Y., Neyshabur, B., and Roelofs, R.
\newblock The evolution of out-of-distribution robustness throughout
  fine-tuning.
\newblock \emph{CoRR}, abs/2106.15831, 2021.

\bibitem[Arjovsky et~al.(2019)Arjovsky, Bottou, Gulrajani, and
  Lopez{-}Paz]{IRM}
Arjovsky, M., Bottou, L., Gulrajani, I., and Lopez{-}Paz, D.
\newblock Invariant risk minimization.
\newblock \emph{CoRR}, abs/1907.02893, 2019.

\bibitem[AutoGluon()]{AutoGluon}
AutoGluon.
\newblock {AutoGluon: AutoML for Text, Image, and Tabular Data}.
\newblock \url{https://auto.gluon.ai/}, 2022.

\bibitem[Badgeley et~al.(2019)Badgeley, Zech, Oakden-Rayner, Glicksberg, Liu,
  Gale, McConnell, Percha, Snyder, and Dudley]{Badgeley2019-al}
Badgeley, M.~A., Zech, J.~R., Oakden-Rayner, L., Glicksberg, B.~S., Liu, M.,
  Gale, W., McConnell, M.~V., Percha, B., Snyder, T.~M., and Dudley, J.~T.
\newblock Deep learning predicts hip fracture using confounding patient and
  healthcare variables.
\newblock \emph{npj Digital Medicine}, 2\penalty0 (1):\penalty0 1--10, April
  2019.

\bibitem[Baek et~al.(2022)Baek, Jiang, Raghunathan, and Kolter]{agreement}
Baek, C., Jiang, Y., Raghunathan, A., and Kolter, Z.
\newblock Agreement-on-the-line: Predicting the performance of neural networks
  under distribution shift.
\newblock \emph{arXiv preprint arXiv:2206.13089}, 2022.

\bibitem[Ben-David et~al.(2010)Ben-David, Blitzer, Crammer, Kulesza, Pereira,
  and Vaughan]{ben2010theory}
Ben-David, S., Blitzer, J., Crammer, K., Kulesza, A., Pereira, F., and Vaughan,
  J.~W.
\newblock A theory of learning from different domains.
\newblock \emph{Machine learning}, 79\penalty0 (1):\penalty0 151--175, 2010.

\bibitem[Blodgett et~al.(2016)Blodgett, Green, and
  O'Connor]{Afican-American-English}
Blodgett, S.~L., Green, L., and O'Connor, B.~T.
\newblock Demographic dialectal variation in social media: {A} case study of
  african-american english.
\newblock In \emph{{EMNLP}}, pp.\  1119--1130. The Association for
  Computational Linguistics, 2016.

\bibitem[Buolamwini \& Gebru(2018)Buolamwini and Gebru]{Gender-Shades}
Buolamwini, J. and Gebru, T.
\newblock Gender shades: Intersectional accuracy disparities in commercial
  gender classification.
\newblock In \emph{{FAT}}, volume~81 of \emph{Proceedings of Machine Learning
  Research}, pp.\  77--91. {PMLR}, 2018.

\bibitem[Chen et~al.(2015)Chen, Li, Li, Lin, Wang, Wang, Xiao, Xu, Zhang, and
  Zhang]{chen2015mxnet}
Chen, T., Li, M., Li, Y., Lin, M., Wang, N., Wang, M., Xiao, T., Xu, B., Zhang,
  C., and Zhang, Z.
\newblock Mxnet: A flexible and efficient machine learning library for
  heterogeneous distributed systems.
\newblock \emph{arXiv preprint arXiv:1512.01274}, 2015.

\bibitem[Daneshjou et~al.(2021)Daneshjou, Vodrahalli, Liang, Novoa, Jenkins,
  Rotemberg, Ko, Swetter, Bailey, Gevaert, Mukherjee, Phung, Yekrang, Fong,
  Sahasrabudhe, Zou, and Chiou]{DDI}
Daneshjou, R., Vodrahalli, K., Liang, W., Novoa, R.~A., Jenkins, M., Rotemberg,
  V.~M., Ko, J.~M., Swetter, S.~M., Bailey, E.~E., Gevaert, O., Mukherjee, P.,
  Phung, M., Yekrang, K., Fong, B., Sahasrabudhe, R., Zou, J., and Chiou, A.~S.
\newblock Disparities in dermatology ai: Assessments using diverse clinical
  images.
\newblock \emph{ArXiv}, abs/2111.08006, 2021.

\bibitem[DeGrave et~al.(2021)DeGrave, Janizek, and Lee]{DeGrave2021-xv}
DeGrave, A.~J., Janizek, J., and Lee, S.-I.
\newblock {AI} for radiographic {COVID-19} detection selects shortcuts over
  signal.
\newblock \emph{Nature Machine Intelligence}, 3\penalty0 (7):\penalty0
  610--619, May 2021.

\bibitem[Eyuboglu et~al.(2022{\natexlab{a}})Eyuboglu, Varma, Saab, Delbrouck,
  Lee-Messer, Dunnmon, Zou, and Re]{Domino}
Eyuboglu, S., Varma, M., Saab, K.~K., Delbrouck, J.-B., Lee-Messer, C.,
  Dunnmon, J., Zou, J., and Re, C.
\newblock Domino: Discovering systematic errors with cross-modal embeddings.
\newblock In \emph{International Conference on Learning Representations},
  2022{\natexlab{a}}.
\newblock URL \url{https://openreview.net/forum?id=FPCMqjI0jXN}.

\bibitem[Eyuboglu et~al.(2022{\natexlab{b}})Eyuboglu, Varma, Saab, Delbrouck,
  Lee-Messer, Dunnmon, Zou, and Re]{eyuboglu2022domino}
Eyuboglu, S., Varma, M., Saab, K.~K., Delbrouck, J.-B., Lee-Messer, C.,
  Dunnmon, J., Zou, J., and Re, C.
\newblock Domino: Discovering systematic errors with cross-modal embeddings.
\newblock In \emph{International Conference on Learning Representations},
  2022{\natexlab{b}}.
\newblock URL \url{https://openreview.net/forum?id=FPCMqjI0jXN}.

\bibitem[Fuster et~al.(2017)Fuster, Goldsmith-Pinkham, Ramadorai, and
  Walther]{ML-Credit-Markets}
Fuster, A., Goldsmith-Pinkham, P., Ramadorai, T., and Walther, A.
\newblock Predictably unequal? the effects of machine learning on credit
  markets.
\newblock \emph{Regulation of Financial Institutions eJournal}, 2017.

\bibitem[Grother et~al.(2011)Grother, Grother, Phillips, and
  Quinn]{Face-Recognition-Fairness}
Grother, P.~J., Grother, P.~J., Phillips, P.~J., and Quinn, G.~W.
\newblock \emph{Report on the evaluation of 2D still-image face recognition
  algorithms}.
\newblock Citeseer, 2011.

\bibitem[Gulrajani \& Lopez-Paz(2020)Gulrajani and Lopez-Paz]{DomainBed}
Gulrajani, I. and Lopez-Paz, D.
\newblock In search of lost domain generalization.
\newblock \emph{arXiv preprint arXiv:2007.01434}, 2020.

\bibitem[Hashimoto et~al.(2018)Hashimoto, Srivastava, Namkoong, and
  Liang]{Fairness-Without-Demographics}
Hashimoto, T.~B., Srivastava, M., Namkoong, H., and Liang, P.
\newblock Fairness without demographics in repeated loss minimization.
\newblock In \emph{{ICML}}, volume~80 of \emph{Proceedings of Machine Learning
  Research}, pp.\  1934--1943. {PMLR}, 2018.

\bibitem[Hendrycks \& Dietterich(2019)Hendrycks and
  Dietterich]{ImageNet-C-Dataset}
Hendrycks, D. and Dietterich, T.~G.
\newblock Benchmarking neural network robustness to common corruptions and
  perturbations.
\newblock In \emph{{ICLR} (Poster)}. OpenReview.net, 2019.

\bibitem[Jurgens et~al.(2017)Jurgens, Tsvetkov, and
  Jurafsky]{language-identification-dialects}
Jurgens, D., Tsvetkov, Y., and Jurafsky, D.
\newblock Incorporating dialectal variability for socially equitable language
  identification.
\newblock In \emph{{ACL} {(2)}}, pp.\  51--57. Association for Computational
  Linguistics, 2017.

\bibitem[Kaplun et~al.(2022)Kaplun, Ghosh, Garg, Barak, and
  Nakkiran]{kaplun2022deconstructing}
Kaplun, G., Ghosh, N., Garg, S., Barak, B., and Nakkiran, P.
\newblock Deconstructing distributions: A pointwise framework of learning.
\newblock \emph{arXiv preprint arXiv:2202.09931}, 2022.

\bibitem[Koenecke et~al.(2020)Koenecke, Nam, Lake, Nudell, Quartey, Mengesha,
  Toups, Rickford, Jurafsky, and
  Goel]{Racial-disparities-in-automated-speech-recognition}
Koenecke, A., Nam, A. J.~H., Lake, E., Nudell, J., Quartey, M., Mengesha, Z.,
  Toups, C., Rickford, J.~R., Jurafsky, D., and Goel, S.
\newblock Racial disparities in automated speech recognition.
\newblock \emph{Proceedings of the National Academy of Sciences of the United
  States of America}, 117:\penalty0 7684 -- 7689, 2020.

\bibitem[Koh et~al.(2021)Koh, Sagawa, Marklund, Xie, Zhang, Balsubramani, Hu,
  Yasunaga, Phillips, Gao, Lee, David, Stavness, Guo, Earnshaw, Haque, Beery,
  Leskovec, Kundaje, Pierson, Levine, Finn, and Liang]{WILDS}
Koh, P.~W., Sagawa, S., Marklund, H., Xie, S.~M., Zhang, M., Balsubramani, A.,
  Hu, W., Yasunaga, M., Phillips, R.~L., Gao, I., Lee, T., David, E., Stavness,
  I., Guo, W., Earnshaw, B., Haque, I., Beery, S.~M., Leskovec, J., Kundaje,
  A., Pierson, E., Levine, S., Finn, C., and Liang, P.
\newblock {WILDS:} {A} benchmark of in-the-wild distribution shifts.
\newblock In \emph{{ICML}}, volume 139 of \emph{Proceedings of Machine Learning
  Research}, pp.\  5637--5664. {PMLR}, 2021.

\bibitem[Kornblith et~al.(2019)Kornblith, Shlens, and
  Le]{ImageNet-Transfer-Better}
Kornblith, S., Shlens, J., and Le, Q.~V.
\newblock Do better imagenet models transfer better?
\newblock In \emph{{CVPR}}, pp.\  2661--2671. Computer Vision Foundation /
  {IEEE}, 2019.

\bibitem[Krueger et~al.(2021)Krueger, Caballero, Jacobsen, Zhang, Binas, Zhang,
  Priol, and Courville]{VREx}
Krueger, D., Caballero, E., Jacobsen, J., Zhang, A., Binas, J., Zhang, D.,
  Priol, R.~L., and Courville, A.~C.
\newblock Out-of-distribution generalization via risk extrapolation (rex).
\newblock In \emph{{ICML}}, volume 139 of \emph{Proceedings of Machine Learning
  Research}, pp.\  5815--5826. {PMLR}, 2021.

\bibitem[Li et~al.(2017)Li, Yang, Song, and Hospedales]{PACS}
Li, D., Yang, Y., Song, Y.-Z., and Hospedales, T.~M.
\newblock Deeper, broader and artier domain generalization.
\newblock \emph{2017 IEEE International Conference on Computer Vision (ICCV)},
  pp.\  5543--5551, 2017.

\bibitem[Liang \& Zou(2022)Liang and Zou]{MetaShift}
Liang, W. and Zou, J.
\newblock Metashift: A dataset of datasets for evaluating contextual
  distribution shifts and training conflicts.
\newblock In \emph{International Conference on Learning Representations}, 2022.
\newblock URL \url{https://openreview.net/forum?id=MTex8qKavoS}.

\bibitem[Liang et~al.(2022)Liang, Tadesse, Ho, Fei-Fei, Zaharia, Zhang, and
  Zou]{liang2022advances}
Liang, W., Tadesse, G.~A., Ho, D., Fei-Fei, L., Zaharia, M., Zhang, C., and
  Zou, J.
\newblock Advances, challenges and opportunities in creating data for
  trustworthy ai.
\newblock \emph{Nature Machine Intelligence}, 4\penalty0 (8):\penalty0
  669--677, 2022.

\bibitem[Liang et~al.(2023)Liang, Yuksekgonul, Mao, Wu, and Zou]{liang2023gpt}
Liang, W., Yuksekgonul, M., Mao, Y., Wu, E., and Zou, J.
\newblock Gpt detectors are biased against non-native english writers.
\newblock \emph{arXiv preprint arXiv:2304.02819}, 2023.

\bibitem[Liu et~al.(2021)Liu, Haghgoo, Chen, Raghunathan, Koh, Sagawa, Liang,
  and Finn]{JTT}
Liu, E.~Z., Haghgoo, B., Chen, A.~S., Raghunathan, A., Koh, P.~W., Sagawa, S.,
  Liang, P., and Finn, C.
\newblock Just train twice: Improving group robustness without training group
  information.
\newblock In \emph{{ICML}}, volume 139 of \emph{Proceedings of Machine Learning
  Research}, pp.\  6781--6792. {PMLR}, 2021.

\bibitem[Lu et~al.(2020)Lu, Nott, Olson, Todeschini, Vahabi, Carmon, and
  Schmidt]{CIFAR-10.2}
Lu, S., Nott, B., Olson, A., Todeschini, A., Vahabi, H., Carmon, Y., and
  Schmidt, L.
\newblock Harder or different? a closer look at distribution shift in dataset
  reproduction.
\newblock In \emph{ICML Workshop on Uncertainty and Robustness in Deep
  Learning}, 2020.
\newblock
  \url{http://www.gatsby.ucl.ac.uk/~balaji/udl2020/accepted-papers/UDL2020-paper-101.pdf}.

\bibitem[Miller et~al.(2020)Miller, Krauth, Recht, and
  Schmidt]{SQuAD-dataset-reconstruction}
Miller, J., Krauth, K., Recht, B., and Schmidt, L.
\newblock The effect of natural distribution shift on question answering
  models.
\newblock In \emph{{ICML}}, volume 119 of \emph{Proceedings of Machine Learning
  Research}, pp.\  6905--6916. {PMLR}, 2020.

\bibitem[Miller et~al.(2021)Miller, Taori, Raghunathan, Sagawa, Koh, Shankar,
  Liang, Carmon, and Schmidt]{accuracy-on-the-line}
Miller, J., Taori, R., Raghunathan, A., Sagawa, S., Koh, P.~W., Shankar, V.,
  Liang, P., Carmon, Y., and Schmidt, L.
\newblock Accuracy on the line: on the strong correlation between
  out-of-distribution and in-distribution generalization.
\newblock In \emph{{ICML}}, volume 139 of \emph{Proceedings of Machine Learning
  Research}, pp.\  7721--7735. {PMLR}, 2021.

\bibitem[Oakden-Rayner et~al.(2020)Oakden-Rayner, Dunnmon, Carneiro, and
  Re]{Oakden-Rayner2019-zv}
Oakden-Rayner, L., Dunnmon, J., Carneiro, G., and Re, C.
\newblock Hidden stratification causes clinically meaningful failures in
  machine learning for medical imaging.
\newblock In \emph{Proceedings of the ACM Conference on Health, Inference, and
  Learning}, CHIL '20, pp.\  151–159, New York, NY, USA, 2020. Association
  for Computing Machinery.
\newblock ISBN 9781450370462.
\newblock \doi{10.1145/3368555.3384468}.
\newblock URL \url{https://doi.org/10.1145/3368555.3384468}.

\bibitem[Pearl et~al.(2000)]{pearl2000models}
Pearl, J. et~al.
\newblock Models, reasoning and inference.
\newblock \emph{Cambridge, UK: CambridgeUniversityPress}, 19\penalty0 (2),
  2000.

\bibitem[Pfohl et~al.(2022)Pfohl, Zhang, Xu, Foryciarz, Ghassemi, and
  Shah]{clinical-DRO}
Pfohl, S.~R., Zhang, H., Xu, Y., Foryciarz, A., Ghassemi, M., and Shah, N.~H.
\newblock A comparison of approaches to improve worst-case predictive model
  performance over patient subpopulations.
\newblock \emph{Scientific reports}, 12\penalty0 (1):\penalty0 1--13, 2022.

\bibitem[Recht et~al.(2018)Recht, Roelofs, Schmidt, and Shankar]{CIFAR-10.1}
Recht, B., Roelofs, R., Schmidt, L., and Shankar, V.
\newblock Do {CIFAR-10} classifiers generalize to cifar-10?
\newblock \emph{CoRR}, abs/1806.00451, 2018.

\bibitem[Recht et~al.(2019)Recht, Roelofs, Schmidt, and Shankar]{ImageNet-V2}
Recht, B., Roelofs, R., Schmidt, L., and Shankar, V.
\newblock Do imagenet classifiers generalize to imagenet?
\newblock In \emph{{ICML}}, volume~97 of \emph{Proceedings of Machine Learning
  Research}, pp.\  5389--5400. {PMLR}, 2019.

\bibitem[Redko et~al.(2020)Redko, Morvant, Habrard, Sebban, and
  Bennani]{redko2020survey}
Redko, I., Morvant, E., Habrard, A., Sebban, M., and Bennani, Y.
\newblock A survey on domain adaptation theory: learning bounds and theoretical
  guarantees.
\newblock \emph{arXiv preprint arXiv:2004.11829}, 2020.

\bibitem[Rolf et~al.(2021)Rolf, Worledge, Recht, and Jordan]{Modified-CIFAR-4}
Rolf, E., Worledge, T.~T., Recht, B., and Jordan, M.~I.
\newblock Representation matters: Assessing the importance of subgroup
  allocations in training data.
\newblock In \emph{{ICML}}, volume 139 of \emph{Proceedings of Machine Learning
  Research}, pp.\  9040--9051. {PMLR}, 2021.

\bibitem[Sagawa et~al.(2020{\natexlab{a}})Sagawa, Koh, Hashimoto, and
  Liang]{GroupDRO}
Sagawa, S., Koh, P.~W., Hashimoto, T.~B., and Liang, P.
\newblock Distributionally robust neural networks.
\newblock In \emph{{ICLR}}. OpenReview.net, 2020{\natexlab{a}}.

\bibitem[Sagawa et~al.(2020{\natexlab{b}})Sagawa, Koh, Hashimoto, and
  Liang]{Sagawa2020-bt}
Sagawa, S., Koh, P.~W., Hashimoto, T.~B., and Liang, P.
\newblock Distributionally robust neural networks for group shifts: On the
  importance of regularization for worst-case generalization.
\newblock In \emph{ICLR}, 2020{\natexlab{b}}.

\bibitem[Sagawa et~al.(2020{\natexlab{c}})Sagawa, Raghunathan, Koh, and
  Liang]{Overparameterization-Exacerbates-Spurious-Correlations}
Sagawa, S., Raghunathan, A., Koh, P.~W., and Liang, P.
\newblock An investigation of why overparameterization exacerbates spurious
  correlations.
\newblock In \emph{{ICML}}, volume 119 of \emph{Proceedings of Machine Learning
  Research}, pp.\  8346--8356. {PMLR}, 2020{\natexlab{c}}.

\bibitem[Santurkar et~al.(2021)Santurkar, Tsipras, and Madry]{BREEDS}
Santurkar, S., Tsipras, D., and Madry, A.
\newblock {BREEDS:} benchmarks for subpopulation shift.
\newblock In \emph{{ICLR}}. OpenReview.net, 2021.

\bibitem[Sapiezynski et~al.(2017)Sapiezynski, Kassarnig, and
  Wilson]{academic-recommender-systems-gender}
Sapiezynski, P., Kassarnig, V., and Wilson, C.
\newblock Academic performance prediction in a gender-imbalanced environment.
\newblock In \emph{FATREC Workshop on Responsible Recommendation}, 2017.

\bibitem[Taori et~al.(2020)Taori, Dave, Shankar, Carlini, Recht, and
  Schmidt]{robustness-natural-shifts}
Taori, R., Dave, A., Shankar, V., Carlini, N., Recht, B., and Schmidt, L.
\newblock Measuring robustness to natural distribution shifts in image
  classification.
\newblock In \emph{NeurIPS}, 2020.

\bibitem[Tatman(2017)]{Youtube-captioning-Gender-Dialect}
Tatman, R.
\newblock Gender and dialect bias in youtube's automatic captions.
\newblock In \emph{EthNLP@EACL}, 2017.

\bibitem[Venkateswara et~al.(2017)Venkateswara, Eusebio, Chakraborty, and
  Panchanathan]{OfficeHome}
Venkateswara, H., Eusebio, J., Chakraborty, S., and Panchanathan, S.
\newblock Deep hashing network for unsupervised domain adaptation.
\newblock In \emph{Proceedings of the IEEE Conference on Computer Vision and
  Pattern Recognition}, pp.\  5018--5027, 2017.

\bibitem[Wu et~al.(2021)Wu, Wu, Daneshjou, Ouyang, Ho, and Zou]{wu2021medical}
Wu, E., Wu, K., Daneshjou, R., Ouyang, D., Ho, D.~E., and Zou, J.
\newblock How medical ai devices are evaluated: limitations and recommendations
  from an analysis of fda approvals.
\newblock \emph{Nature Medicine}, 27\penalty0 (4):\penalty0 582--584, 2021.

\bibitem[Yadav \& Bottou(2019)Yadav and Bottou]{MNIST-dataset-reconstruction}
Yadav, C. and Bottou, L.
\newblock Cold case: The lost {MNIST} digits.
\newblock In \emph{NeurIPS}, pp.\  13443--13452, 2019.

\bibitem[Yao et~al.(2022)Yao, Wang, Li, Zhang, Liang, Zou, and Finn]{huaxiu}
Yao, H., Wang, Y., Li, S., Zhang, L., Liang, W., Zou, J., and Finn, C.
\newblock Improving out-of-distribution robustness via selective augmentation.
\newblock In \emph{Proceeding of the Thirty-ninth International Conference on
  Machine Learning}, 2022.

\bibitem[Zech et~al.(2018)Zech, Badgeley, Liu, Costa, Titano, and
  Oermann]{zech2018variable}
Zech, J.~R., Badgeley, M.~A., Liu, M., Costa, A.~B., Titano, J.~J., and
  Oermann, E.~K.
\newblock Variable generalization performance of a deep learning model to
  detect pneumonia in chest radiographs: a cross-sectional study.
\newblock \emph{PLoS medicine}, 15\penalty0 (11):\penalty0 e1002683, 2018.

\end{thebibliography}
\bibliographystyle{icml2023}

\newpage
\appendix
\onecolumn
\newpage
\clearpage
\newpage
\onecolumn

\section{Extended Description of Experiment Setups}

\subsection{Subpopulation Shift Datasets}

We categorize our subpopulation shift datasets based on the underlying reason that ML models exhibits degraded performance on the minority subpopulation. 
Based on our survey and prior work~\cite{Domino,Oakden-Rayner2019-zv}, we identified two important scenarios and the corresponding datasets as follows:
\begin{itemize}
\itemsep0em

    \item 
    \textbf{Spurious correlation.} 
    In statistics, a spurious correlation refers to a connection between two variables that appear to be causal but are not. 
    Take the Cat vs. Dog task as an example: With cat images mostly indoor and dog images mostly outdoor, A spurious correlation exists between the class labels and the indoor/outdoor contexts~\citep{MetaShift}.
    To explore the scenario of spurious correlation, we experiment with three existing datasets in the community: MetaShift~\citep{MetaShift}, Waterbirds~\citep{Sagawa2020-bt}, and {Modified-CIFAR4 V1}~\citep{Modified-CIFAR-4}. The detailed setups are as follows:
    
    \begin{itemize}
        \item \textbf{MetaShift-Cat-Dog}~\cite{MetaShift}: A cat vs. dog binary classification task, where cat images are mostly indoor and dog images are mostly outdoor. We choose indoor context as the majority subpopulation and outdoor context as the minority subpopulation.
        In the training set, with spurious correlation, $88\%$ of the cat images are \emph{indoors}, and $88\%$ of the dog images are \emph{outdoors}. 
        In the ID test set, the two subpopulations are in the same proportion as the train set;
        In the OOD test set, the two subpopulations are equally represented. 
        Same for the following datasets.
        
        \item \textbf{Waterbirds}~\cite{Sagawa2020-bt}: A waterbird vs. landbird binary classification task, with waterbirds (landbirds) more frequently appearing against a water (land) background in the training distribution. Here the majority subpopulation is water background and the minority subpopulation is land background. We select $80\%$ of the waterbird images with water background, and $80\%$ of the landbird images with land background.
        
        \item \textbf{Modified-CIFAR4 V1}: Created by ~\cite{Modified-CIFAR-4} based on CIFAR-10, a binary classification task to predict whether the image subject moves primarily by air (plane/bird) or land (car/horse), which is spuriously correlated with whether the image contains an animal (bird/horse) or vehicle (car/plane). 
        Here the majority subpopulation is vehicle domain and the minority subpopulation is animal domain. We select $90\%$ of the air images as “air-vehicle(airplane)”, and $90\%$ of the land images as “land-animal(horse)”.
        
    \end{itemize}
    
    \item 
    \textbf{Rare subpopulation.}
    Without spurious correlation, ML models can still underperform on subpopulation that
    occur infrequently in the training set (e.g. patients with a darker skin tone, photos taken at night). 
    Since the rare subpopulation will not significantly affect model loss during training, the model may fail to learn to classify examples within theis subpopulation. 
    To explore the rare subpopulation scenario (without spurious correlation), we adopted {PACS}~\cite{PACS}, {OfficeHome}~\cite{OfficeHome}, {Modified-CIFAR4 V2}~\cite{Modified-CIFAR-4} for the experiments.
    For each dataset, we select 2 domains independent of the target $Y$.
    For example, for PACS, we construct a training set where most of the training images ($67\%$) come from the photo domain, and the rest ($33\%$) come from the cartoon domain, but the domain does not correlate with the class label.
    
    \begin{itemize}
        \item \textbf{PACS}~\cite{PACS}: An image classification task of 7 classes.
        Most of the training images (1200 images, $67\%$) come from the photo domain, and the rest (600 images, $33\%$) come from the cartoon domain. 
        Here the majority subpopulation is the photo domain, and the minority subpopulation is the cartoon domain. 
        \item \textbf{OfficeHome}~\cite{OfficeHome}: An object recognition task of 65 classes. 
        Most of the training images (3965 images, $83\%$) come from the product domain, and the rest (800 images, $17\%$) come from the clipart domain. 
        Here the majority subpopulation is the product domain, and the minority subpopulation is the clipart domain. 
        
        \item \textbf{Modified-CIFAR4 V2}~\cite{Modified-CIFAR-4}: A binary classification task on air vs. land. The majority subpopulation (4500 images, $90\%$ in training) is “air-vehicle(airplane)”,  “land-vehicle(automobile)”, and the minority subpopulation (500 images, $10\%$ in training) is “air-animal(bird)”, “land-animal(horse)” . 
        
    \end{itemize}

\end{itemize}

\subsection{Configuration space of different ML models}

The goal of our paper is to compare the ID and OOD performance accross a wide spectrum of ML models. Therefore, for each subpopulation shifts dataset, we vary the model architectures, training durations, and hyperparameters to explore the performances of different ML models.

Specifically, we follow the search space of AutoGluon, the state-of-the-art commercial AutoML library~\cite{AutoGluon}, to train 500 different ML models with varying training settings.
For each dataset, we implement with 5 model architectures, namely mobilenetv3\_small, resnet18\_v1b, resnet50\_v1, mobilenetv3\_large, and resnet101\_v2. We set the learning rates to be in the search space of $\{0.0001, 0.0005, 0.001, 0.005, 0.01\}$ and batch sizes in $\{8, 16, 32, 64, 128\}$. The training durations are set to $\{1, 5, 10, 25\}$ epochs, as we did not find significant improvements from training for more epochs for any dataset. Together with the 5 model architectures, 5 learning rates, 5 batch sizes, and 4 training durations, we train 500 different ML models independently with all the combinations of different model architectures and hyperparameters.

\subsection{Implementation and Reproducibility}

Our implementation framework is based on MXNet~\cite{chen2015mxnet} and AutoGluon~\cite{AutoGluon}. Following the training configurations in AutoGluon, we adopt the Nesterov accelerated gradient optimizer with momentum=0.9. Other unspecified parameters are set to the AutoGluon defaults. There is no learning rate decay. All ML models are fine-tuned starting from its ImageNet pre-trained checkpoints. 
Code and data are available at \url{https://github.com/yining-mao/Moon-Shape-ICML-2023}

\subsection{Controlled Experiments on Spurious Correlation}
\label{appendix-subsec:controlled-experiments}

\paragraph{Setup Details}
The Modified-CIFAR4 (V1) in Figure~\ref{fig:figure-spurious} was created by ~\citep{Modified-CIFAR-4} based on CIFAR-10 by subsetting to the bird, car, horse, and plane classes. 
This is a binary classification task to predict whether the image subject moves primarily by air (airplane/bird) or land (automobile/horse), which is spuriously correlated with 2 domains where the image contains an animal (bird/horse) or vehicle (automobile/airplane). Here the majority subpopulation is vehicle domain as indicated by $Z=1$ in Figure~\ref{fig:figure-spurious}; the minority subpopulation is animal domain as indicated by $Z=0$ in Figure~\ref{fig:figure-spurious}.

We modulate the degree of \emph{spurious correlations} between the classification target label (air/land) and spurious feature (vehicle/animal) in the training data by changing the \emph{mixture weights} in the training data. The $2\times2$ table in each panel in Figure~\ref{fig:figure-spurious} indicates the dataset construction procedure. 
Specifically, we ensure that the dataset is class-balanced: i.e., 5,000 images for both $Y=0$ and $Y=1$, and fixed the ratio of spurious feature vehicle $Z=1$ ($60\%$ in training, i.e., 6,000 images) and spurious feature animal $Z=0$ ($40\%$ in training, i.e., 4,000 images). 
We increase the level of the spurious correlation between the classification target label (air/land) and spurious feature (vehicle/animal) by increasing the number of samples in air-vehicle(airplane)” and “land-animal(horse)” as indicated by the red boxes in Figure~\ref{fig:figure-spurious}.
Formally, since we fix (1) the total number of data points in the training set ($10,000$), (2) the ratio of data point number in each class ($P(Y=1)=0.5, P(Y=0)=0.5$), and (3) the ratio of data point number in each spurious features ($P(Z=1)=0.6, P(Z=0)=0.4$), there is effectively only one degree of freedom left, which we vary to change the level of spurious correlation.

\section{Additional Experimental Results}
\textbf{Training dynamics.} We show that the moon shape persists both across and within different training epochs in \textbf{Figure~\ref{fig:figure-persist} (b)}. Results on other datasets are similar as shown in \textbf{Supp. Figure~\ref{fig:training-dynamics-other-datasets}}. 
 \begin{figure*}[tb]
    \centering
    \includegraphics[width=0.75\textwidth]
    {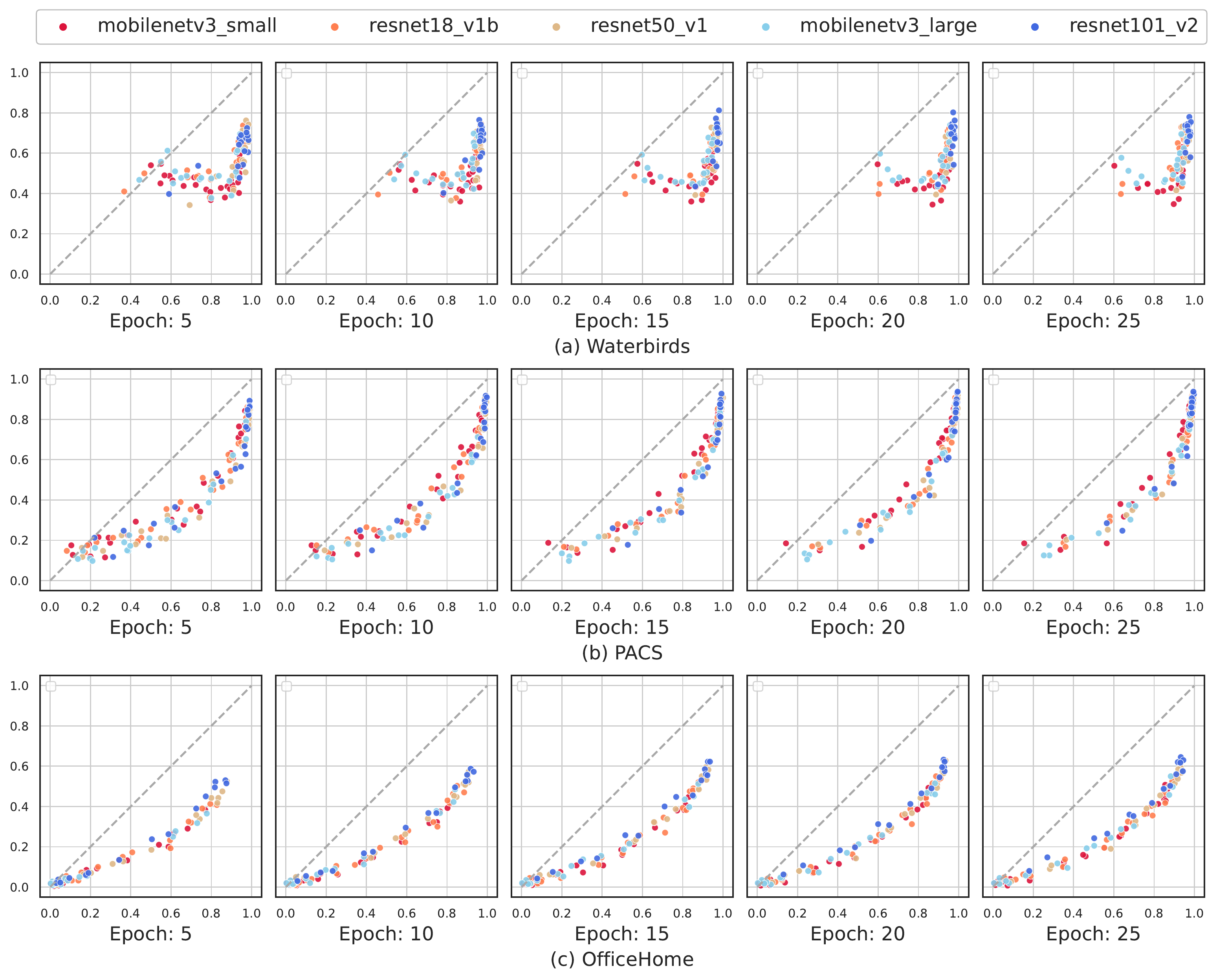}
    \caption{
    \small
    \textbf{The moon shape persists across different training epochs.
    Results on other datasets similar to Figure~\ref{fig:figure-persist}. 
    }
    We stratify Figure~\ref{fig:figure2} based on the number of training epochs. 
    The x-axis indicates majority subpopulation performance. 
    The y-axis indicates minority subpopulation performance. 
    Most of the models have converged after 10 epochs. 
    The moon shape is apparent in each snapshot and persists across training epochs.  
    }
    \label{fig:training-dynamics-other-datasets}
\end{figure*}

\section{Extended Related Work}

\paragraph{Linear correlations between ID and OOD performances}
Existing research mostly reports \emph{linear} correlations between ID and OOD performances. 
The linear correlations were first reported in recent dataset reconstruction settings including ImageNet-V2~\cite{ImageNet-V2}, CIFAR-10.1~\cite{CIFAR-10.1}, CIFAR-10.2~\cite{CIFAR-10.2}, 
where new test sets of popular benchmarks are collected closely following the original dataset creation process. 
As there are subtle differences in the dataset creation pipeline, the test performance on the new test set is often lower, but appears to be linearly correlated with the performance on the original test set~\cite{CIFAR-10.2,SQuAD-dataset-reconstruction,CIFAR-10.1,ImageNet-V2,MNIST-dataset-reconstruction}. 
Later researchers also found the linear trends in the context of cross-benchmark evaluation~\cite{robustness-natural-shifts,accuracy-on-the-line}, 
and transfer learning~\cite{ImageNet-Transfer-Better,OOD-Fine-Tuning}, where a model's ImageNet test accuracy linearly correlates with the transfer learning accuracy. 
Similar linear trends are also observed in sub-type shifts~\cite{ImageNet-C-Dataset,BREEDS} (e.g., the training data for the “dog” class are all from a specific breed while the test data come from another breed). 
Different from these studies, we (1) focus on subpopulation shifts, where we also present the first systematic study on the performance correlation between data subpopulations, and (2) find \emph{nonlinear} correlations of ML performance across data subpopulations, which is not captured in previous work. 
Importantly, we show that for datasets with spurious correlations, even with probit-transformed axes as used by several prior work~\cite{accuracy-on-the-line,ImageNet-V2,robustness-natural-shifts}, the performance correlations still remain nonlinear. This confirms that the nonlinear (“moon shape”) correlation phenomenon is indeed not captured by previous work.

\paragraph{ML with diverse subpopulations} 
A major challenge in ML is that a model can have very disparate performances even when it's applied to different subpopulations of its training and evaluation data. 
Models with low average error can still fail on particular groups of data points~\cite{Fairness-Without-Demographics,Gender-Shades,Afican-American-English}. 
For example, 
predictive models for clinical outcomes that are accurate on average in a patient population are reported to underperform drastically for some subpopulations, potentially introducing or reinforcing inequities in care access and quality~\cite{clinical-DRO}. 
Similar performance disparity, have also been observed in radiograph classification~\cite{Badgeley2019-al, zech2018variable, DeGrave2021-xv}, 
face recognition~\cite{Face-Recognition-Fairness,Gender-Shades}, 
speech recognition~\cite{Racial-disparities-in-automated-speech-recognition,Afican-American-English,language-identification-dialects}, 
academic recommender systems~\cite{academic-recommender-systems-gender}), 
and
automatic video captioning~\cite{Youtube-captioning-Gender-Dialect}, among others. 
Worse, as model accuracy affects user retention, the minority group might shrink and thus even amplifies the performance disparity over time~\cite{Fairness-Without-Demographics,ML-Credit-Markets}.  
These case studies highlight the importance of understanding the ML performance disparity across subpopulations.

\section{Extended Analysis: Simulation Studies}

\newcommand{\nmin}{n_\mathsf{min}}
\newcommand{\nmaj}{n_\mathsf{maj}}
\newcommand{\ntot}{n}
\newcommand{\xcau}{x_\mathsf{core}}
\newcommand{\xspu}{x_\mathsf{spu}}

\newcommand{\dcau}{d_\mathsf{core}}
\newcommand{\dspu}{d_\mathsf{spu}}

\newcommand\sN{\ensuremath{\mathcal{N}}}
\newcommand\R{\ensuremath{\mathbb{R}}} %
\newcommand{\sigcau}{\sigma_\mathsf{core}}
\newcommand{\sigspu}{\sigma_\mathsf{spu}}

\newcommand{\pmajmath}{p_\mathsf{maj}}
\newcommand{\scrmath}{r_\mathsf{s:c}}
\newcommand{\scdmath}{d_\mathsf{s:c}}

\paragraph{Simulation Data distribution.}

We construct a synthetic dataset that replicates the "moon shape" phenomenon in section 3. The label $ y \in \{1,-1\} $ is spuriously correlated with a spurious attribute $ a \in \{1,-1\} $.

Inspired by \cite{Overparameterization-Exacerbates-Spurious-Correlations}, 
we divide our training data into two groups: majority group of size $\nmaj$ with $a=y$, and minority group of size $\nmin$ with $a=-y$. Both of majority and minority group have their own distribution over input features $x = [\xcau, \xspu] \in \R^{\dcau+\dspu}$ comprising core features $\xcau \in \R^{\dcau}$  generated from the label/core attribute $y$, and spurious features $\xspu \in \R^{\dspu}$ generated from the spurious attribute $a$, and the core and spurious features are both balanced in each group:
\begin{align}
  \label{eqn:toy1}
  \xcau \mid y \sim \sN(y\mathbf{1},\sigcau^2I_d)\nonumber\\
  \xspu \mid a \sim \sN(a\mathbf{1},\sigspu^2I_d).
\end{align}

We fix (1) the total number of training points $n$ as 3000, (2) $\sigcau=10$, (3) $\sigspu=1$, and (4) the training model as logistic regression model with different hyper parameters to train on the synthetic dataset. We test the performance on the majority group and the minority group separately, and conduct several controlled experiments to study the moon shape phenomenon and the effect of the spurious correlation on synthetic dataset.

  \begin{figure*}[tb]
    \centering
    \includegraphics[width=0.45\textwidth]
    {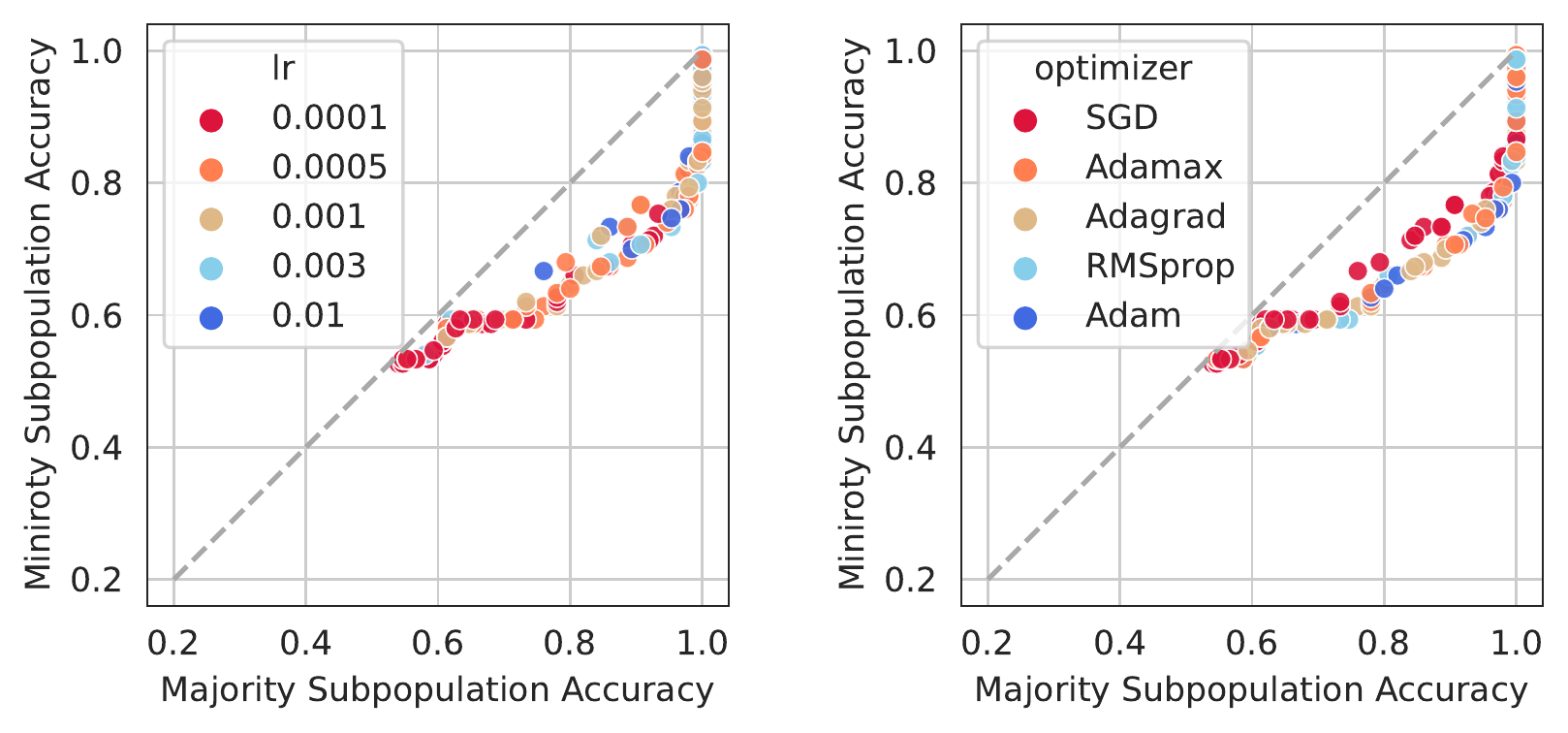}
    \caption{
    \small
    \textbf{We present a simple simulation study which successfully replicates the moon shape phenomenon.} 
    Moon-shape trend holds in synthetic dataset and persists across hyper parameters.
    }
    \label{fig:figure_moonshape}
\end{figure*}

\paragraph{Experiments and Results}  
We find the moon shape trend as we observe on real datasets, which indicates the crucial role of the existence of spurious correlation in shaping the moon shape trend. And the moon shape correlation holds across hyperparameters. (Figure~\ref{fig:figure_moonshape})

\noindent
\begin{minipage}[t]{0.75\textwidth}

We draw the dots when the models have converged in Figure~\ref{fig:figure_epoch}. Different from the real datasets, the dots do not perform the moon shape when the models have converged, but centralize at the top right corner under the synthetic dataset setting since both of majority and minority accuracy have converged to 1. And there are also some dots overlapping with the same majority and minority accuracy, which makes the dots fewer than the number of trials. This shows the difference between the synthetic datasets and the real datasets.

In Figure~\ref{fig:figure_spurious} , we conduct multiple controlled experiments on the synthetic dataset to explore the effect of the level of spurious correlation on moon shape, which is discussed in section 3.2. We define \emph{spurious-core dimension ration} (SDR)  as $\scdmath = \dspu/\dcau$, and define $\ntot = \nmaj + \nmin$ as the total number of training points, with $\pmajmath = \nmaj/\ntot$ as the ratio of majority group. Both of SDR and $\pmajmath$ can reflect the level of the spurious correlation. The higher it is, the stronger spurious correlation there is in the train dataset. The results show the increasing of spurious correlation results in the increasing curvature in the moon shape. 

\end{minipage}\hfill
\begin{minipage}[t]{0.23\textwidth}
  \centering\raisebox{\dimexpr \topskip-\height}{%
  \includegraphics[width=0.95\textwidth]{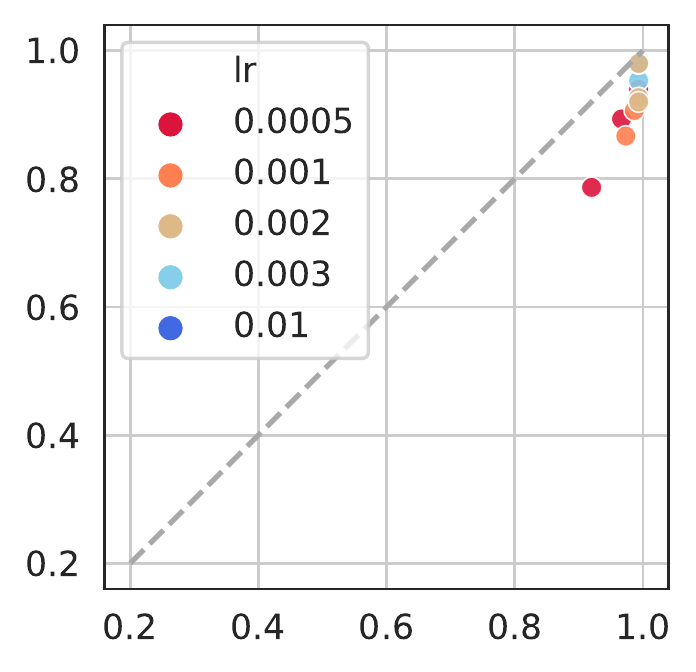}}
  \captionsetup{type=figure}
  \captionof{figure}{
  \small
  The dots centralize in the top right corner when the models have converged in the synthetic dataset.
  }
  \label{fig:figure_epoch}
\end{minipage}

 \begin{figure*}[tb]
    \centering
    \includegraphics[width=0.75\textwidth]
    {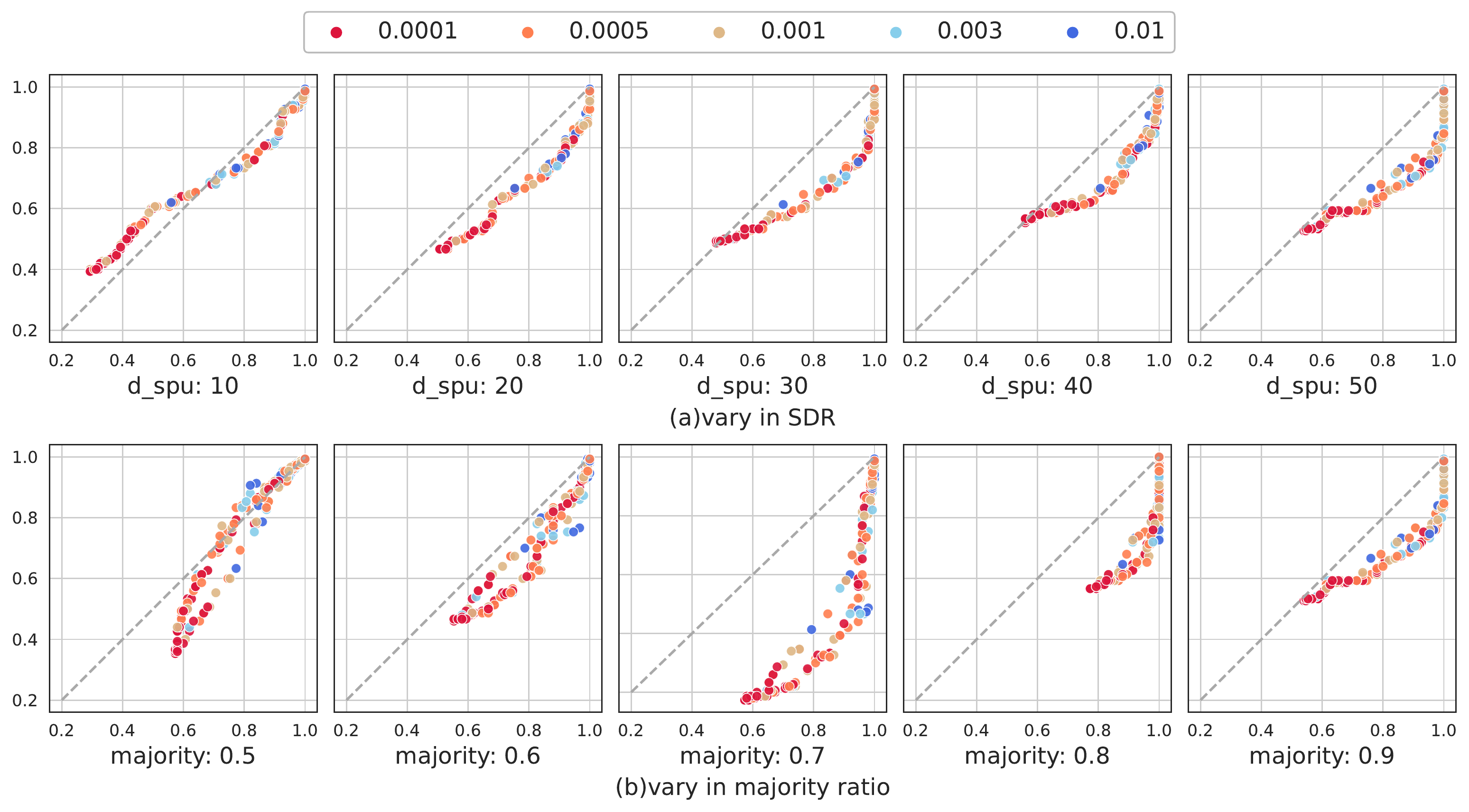}
    \caption{
    \small
    \textbf{Stronger spurious correlation creates increasing curvature in the moon shape.}
    \textbf{Series(a):} $d_{core}$ is fixed as 100, and $d_{spu}$ is varied from 10 to 50, simulating the increasing of SDR.
    \textbf{Series(b):} the ratio of the majority group is varied from 0.5 to 0.9.
    The figures plotted from left to right simulates the increasing of the spurious correlation.
    }
    \label{fig:figure_spurious}
\end{figure*}

\paragraph{Summary} 
In this section, we present a simple simulation study which successfully replicates the moon shape phenomenon in section 3. 
In addition, stronger spurious correlation again leads to more non-linear performance correlations. 
This indicates that the moon shape phenomenon we found might be a very general phenomenon. 
However, one distinct difference between the simulation study and our results on the real world datasets is on the training dynamics: On real-world datasets, we found that the moon shape persists within and across different training epochs (Figure~\ref{fig:figure-persist}), while in the simulation study, models converges to the top-right corner (Supp. Figure~\ref{fig:figure-persist}). 
Our finding on the real-world datasets motivates us to focus our analysis on comparing across \emph{different models} rather than comparing the subpopulation performance of a single model across training epochs. 
Meanwhile, our simulation study indicates that comparing the subpopulation performance of a single model across training epochs is an interesting direction of future work that is complementary to our scope of our main paper.

\newpage

\section{Theoretical analysis of the accuracy gap across subpopulations}
We rigorously study the effect of the spurious correlation on the accuracy gap between the majority and minority subpopulations in a binary classification setting. Our theoretical result shows that the accuracy gap becomes larger when there is a strong spurious correlation between subpopulations and labels and explain why multiple models form the moon shape curve through the ROC curve analysis. %
This is aligned with the theoretical models in the previous literature that spurious correlations can lead to poor accuracy in minority groups~\citep{Overparameterization-Exacerbates-Spurious-Correlations}. 

For $d_X \in \mathbb{N}$, we denote an input space by $\mathcal{X} \subseteq \mathbb{R}^{d_X}$ and denote an input and an output random variable by $X$ and $Y$, respectively. We denote a random variable for a subpopulation by $Z$, where $Z=1$ indicates the majority subpopulation\footnote{Whether $\mathbb{P}(Z=1) \geq 1/2$ or not is not critical in our theoretical analysis, but in terms of our notations, one sufficient condition for $\mathbb{P}(Z=1) \geq 1/2$ is $\pi_1 \geq 1/2$ and $\pi_0 \geq 1/2$.}, otherwise $Z=0$. We assume that the underlying data generating mechanism is $Z \rightarrow Y \rightarrow X$. This implies that $X$ and $Z$ are conditionally independent given $Y$, \textit{i.e.}, $X \perp Z \mid Y$. 
We suppose a conditional distribution of $X$ given $Y$ is given as follows.
\begin{align*}
    X \mid Y= 0 \sim F_0, \quad 
    X \mid Y= 1 \sim F_1,
\end{align*}
for some arbitrary distributions $F_0$ and $F_1$, and we assume 
\begin{align*}
    \mathbb{P}(Z=1 \mid Y=1) = \pi_1, \quad \mathbb{P}(Z=1 \mid Y=0) = \pi_0.
\end{align*}
It is noteworthy that $\pi_1=\pi_0$ is a necessary and sufficient condition for $Y \perp Z$ because $Y$ and $Z$ are Bernoulli random variables. In this respect, the level of spurious correlation can be expressed as $|\pi_1 - \pi_0|$. With these notations, the subpopulation accuracy gap for a model $g:\mathcal{X}\to \{0,1\}$ is expressed as follows:
\begin{align*}
\left| \mathbb{E}[ \mathds{1} (Y = g(X))\mid Z=1]-\mathbb{E}[ \mathds{1} (Y = g(X))\mid Z=0] \right|.
\end{align*}

In the following theorem, we explicitly show that the subpopulation accuracy gap is proportional to the level of spurious correlation $|\pi_1-\pi_0|$.

\begin{theorem}[The higher the level of spurious correlation, the larger the accuracy gap]
The subclass accuracy gap for a classifier $g:\mathcal{X} \to \{0, 1\}$ is expressed as follows.
\begin{align*}
    \textrm{Accuracy Gap} = \frac{\mathbb{P}(Y=1)\mathbb{P}(Y=0)}{\mathbb{P}( Z=1)\mathbb{P}(Z=0)}  |\pi_1 -\pi_0|  \left| \mathrm{TPR} - \mathrm{TNR} \right|,
\end{align*}
where $\mathrm{TPR}$ and $\mathrm{TNR}$ denote the true positive rate $\mathbb{E}( g(X)=1 \mid Y=1)$ under $F_1$ and the true negative rate $\mathbb{E}( g(X)=0 \mid Y=0)$ under $F_0$, respectively.
\label{thm:accuracy_gap}
\end{theorem}
Theorem~\ref{thm:accuracy_gap} shows that the subpopulation accuracy gap is expressed as a function of $|\pi_1 -\pi_0|$ and $\left| \mathrm{TPR} - \mathrm{TNR} \right|$. A direct consequence is that the accuracy gap gets larger when the level of spurious correlation $|\pi_1 -\pi_0|$ increases. It is possible to keep $\mathbb{P}(Z=1)$ and $\mathbb{P}(Y=1)$ as constants while the spurious correlation $|\pi_1 -\pi_0|$ changes. In particular, it occurs when $\pi_1$ and $\pi_0$ are related as $\pi_1= (\mathbb{P}(Z=1)-\mathbb{P}(Y=0)\pi_0)/\mathbb{P}(Y=1)$, which captures the setting of  Figure~\ref{fig:figure-spurious}. Specifically, $\mathbb{P}(Y=1)$ and $\mathbb{P}(Z=1)$ are fixed to $0.5$ and $0.6$ respectively, yet the experimental result shows that the accuracy gap increases once $|\pi_1 -\pi_0|$ increases, which is supported by our theoretical result.

\begin{remark}[Models on the similar ROC curve]
Suppose that there is a trained binary classification model and its ROC curve is not a straight line, which is typically the case. We can think of different points on the ROC curve as different models whose predicted probability outputs are only different by constant shifts. Given that a point on the ROC curve is described as $(1-\mathrm{TNR}, \mathrm{TPR})$, $\mathrm{TPR}$ changes nonlinearly with respect to $\mathrm{TNR}$. Hence, the $|\mathrm{TPR}-\mathrm{TNR}|$ changes nonlinearly, and so does the accuracy gap by Theorem~\ref{thm:accuracy_gap}. This can provide one explanation for our experimental observations that different models form the moon shape curve.
\end{remark}
The setting considered in this remark is admittedly simplified to provide some intuition. In practice, different models (with different architectures and hyperparameters) may not correspond to different points on one ROC curve. However, if the different models do approximately trace out an ROC curve, then the intuition here can apply.

\subsection{Proof of Theorem~\ref{thm:accuracy_gap}} 
\begin{proof}[Proof of Theorem~\ref{thm:accuracy_gap}]
For any $z \in \{0,1\}$, we have 
\begin{align*}
    \mathbb{E}[\mathds{1}(Y=g(X)) \mid Z=z] &= \sum_{y=0} ^1 \mathbb{E}[ \mathds{1}(y=g(X)) \mid Z=z, Y=y] \mathbb{P}(Y=y \mid Z=z)\\
    &=\sum_{y=0} ^1 \mathbb{E}[ \mathds{1}(y=g(X)) \mid Y=y] \mathbb{P}(Y=y \mid Z=z)\\
    &=\mathrm{TPR} \times \mathbb{P}(Y=1 \mid Z=z) + \mathrm{TNR} \times \mathbb{P}(Y=0 \mid Z=z).
\end{align*}
Here, the second equality is due to $X \perp Z \mid Y$. Therefore, the accuracy gap between the two subpopulations is expressed as follows.
\begin{align*}
    \textrm{Accuracy Gap} &= \left| \mathbb{E}[\mathds{1}(Y=g(X)) \mid Z=1]- \mathbb{E}[\mathds{1}(Y=g(X)) \mid Z=0] \right|\\
    &= \Big| \mathrm{TPR} \times \left( \mathbb{P}(Y=1 \mid Z=1) -\mathbb{P}(Y=1 \mid Z=0) \right) \\
    &+ \mathrm{TNR} \times \left( \mathbb{P}(Y=0 \mid Z=1) - \mathbb{P}(Y=0 \mid Z=0) \right) \Big| \\
    &= \left| \mathbb{P}(Y=1 \mid Z=1) -\mathbb{P}(Y=1 \mid Z=0) \right| \times \left| \mathrm{TPR} - \mathrm{TNR}\right|.
\end{align*}
By the Bayes' theorem
\begin{align*}
    \mathbb{P}(Y=1 \mid Z=1) = \frac{\pi_1 \mathbb{P}(Y=1) }{\mathbb{P}(Z=1)}, \quad
    \mathbb{P}(Y=1 \mid Z=0) = \frac{ (1-\pi_1) \mathbb{P}(Y=1) }{\mathbb{P}(Z=0)},
\end{align*}
we have
\begin{align*}
    \mathbb{P}(Y=1 \mid Z=1) -\mathbb{P}(Y=1 \mid Z=0) &= \frac{\pi_1 \mathbb{P}(Y=1) }{\mathbb{P}(Z=1)} - \frac{ (1-\pi_1) \mathbb{P}(Y=1) }{\mathbb{P}(Z=0)} \\
    &= \frac{ \pi_1 \mathbb{P}(Y=1) \mathbb{P}(Z=0) - (1-\pi_1) \mathbb{P}(Y=1) \mathbb{P}(Z=1) }{\mathbb{P}(Z=1)\mathbb{P}(Z=0)} \\
    &= \frac{ \left\{ \mathbb{P}(Z=0) - (1-\pi_1) \right\} \mathbb{P}(Y=1) }{\mathbb{P}(Z=1)\mathbb{P}(Z=0)}.
\end{align*}
Since $\mathbb{P}(Z=0)=1- (\pi_1 \mathbb{P}(Y=1)+\pi_0 \mathbb{P}(Y=0)) = 1 - \pi_1 + (\pi_1 - \pi_0)\mathbb{P}(Y=0)$, we have
\begin{align*}
    \textrm{Accuracy Gap} &=  \frac{\mathbb{P}(Y=1)\mathbb{P}(Y=0)}{\mathbb{P}(Z=1)\mathbb{P}(Z=0)} |\pi_1 -\pi_0| \times \left| \mathrm{TPR} - \mathrm{TNR} \right|.
\end{align*}
It concludes a proof.
\end{proof}

\end{document}